\documentclass[lettersize,journal]{IEEEtran}
\usepackage[T1]{fontenc}
\usepackage{amsmath,amsfonts}
\usepackage{amssymb}
\usepackage{algorithmic}
\usepackage{algorithm}
\usepackage{array}
\usepackage[caption=false,font=normalsize,labelfont=sf,textfont=sf]{subfig}
\usepackage{textcomp}
\usepackage{stfloats}
\usepackage{url}
\usepackage{verbatim}
\usepackage{graphicx}
\usepackage{cite}
\usepackage{multirow}
\usepackage{booktabs}
\usepackage{xcolor}
\usepackage{amsthm}
\newtheorem{definition}{Definition}[]
\newtheorem{problem}{Problem}[]
\newtheorem{theorem}{Theorem}[]

\begin{document}

\title{Dynamic Location Search for Identifying Maximum Weighted Independent Sets in Complex Networks}

%\title{DynLS: An Efficient Algorithm for Solving the Maximum Weighted Independent Set Problem in %Complex Networks}


\author{Enqiang Zhu, Chenkai Hao, Chanjuan Liu, Yongsheng Rao*
        % <-this % stops a space
        


\thanks{The work is supported by the National Natural Science Foundation
of China under Grant 62272115. (Corresponding author: Yongsheng Rao.)}
\thanks{Enqiang Zhu and Yongsheng Rao are with the Institute of Computing Science and Technology, Guangzhou University, Guangzhou {\rm 510006}, China (email: zhuenqiang@gzhu.edu.cn; rysheng@gzhu.edu.cn)}
\thanks{Chenkai Hao is with Cyberspace Institute of Advanced Technology, Guangzhou University, Guangzhou {\rm 510006}, China (email: haochenkai@e.gzhu.edu.cn)}
\thanks{Chanjuan Liu is with the School of Computer Science and Technology, Dalian University of Technology, Dalian {\rm 116024}, China (email: chanjuanliu@dlut.edu.cn)}
}



\maketitle

\begin{abstract}

While Artificial intelligence (AI), including Generative AI, are effective at generating high-quality traffic data and optimization solutions in intelligent transportation systems (ITSs), these techniques often demand significant training time and computational resources, especially in large-scale and complex scenarios. To address this, we introduce a novel and efficient algorithm for solving the maximum weighted independent set (MWIS) problem, which can be used to model many ITSs applications, such as traffic signal control and vehicle routing.
Given the NP-hard nature of the MWIS problem, our proposed algorithm, DynLS, incorporates three key innovations to solve it effectively. First, it uses a scores-based adaptive vertex perturbation (SAVP) technique to accelerate convergence, particularly in sparse graphs. Second, it includes a region location mechanism (RLM) to help escape local optima by dynamically adjusting the search space. Finally, it employs a novel variable neighborhood descent strategy, ComLS, which combines vertex exchange strategies with a reward mechanism to guide the search toward high-quality solutions.
Our experimental results demonstrate DynLS's superior performance, consistently delivering high-quality solutions within 1000 seconds. DynLS outperformed five leading algorithms across 360 test instances, achieving the best solution for 350 instances and surpassing the second-best algorithm, Cyclic-Fast, by 177 instances. Moreover, DynLS matched Cyclic-Fast's convergence speed, highlighting its efficiency and practicality. This research represents a significant advancement in heuristic algorithms for the MWIS problem, offering a promising approach to aid AI techniques in optimizing intelligent transportation systems.
\end{abstract}


\section{Introduction}\label{introduction}
\IEEEPARstart{W}{ith} ongoing advancements in intelligent transportation systems (ITSs), the role of Artificial intelligence (AI), including Generative AI technology, has gained significant prominence \cite{10444919}. This technology offers innovative approaches and methodologies to address challenges such as traffic congestion, route planning, and traffic management \cite{10778265}. While this technique excels at producing high-quality traffic data and optimization solutions, it often necessitates lengthy training periods and considerable computational resources, especially when tackling complex optimization problems \cite{chen2022evolutionary}. Additionally, various optimization challenges may require specific model adjustments and fine-tuning tailored to particular scenarios. As a result, there is a pressing need to develop solutions that can enhance the capabilities of generative AI in handling diverse optimization problems within intelligent transportation systems.

In the intricate landscape of ITSs, numerous key issues can be framed as graph theory problems, such as the maximum independent set (MIS) problem \cite{wang2024joint}. Efficiently resolving these problems is vital for optimizing the structure of traffic networks and enhancing the allocation of traffic resources. For example, when optimizing traffic signal control, selecting non-conflicting combinations of signal phases to maximize traffic flow can be compared to the MIS problem \cite{10404636}. Likewise, when planning vehicle routes \cite{haba2025routing}, identifying the optimal paths to circumvent congested roads is closely related to solutions to the MIS problem. Roughly speaking, the MIS problem is to identify a subset of vertices in a graph where no two vertices are adjacent. 
This problem is well-known for being \textbf{NP}-hard and has been extensively studied due to its significance in both theoretical aspects and practical applications \cite{andrist2023hardness}. 

Recently, a growing focus has been on the MIS problem in weighted graphs, where vertices are associated with positive weights. This gives rise to the maximum weighted independent set (MWIS) problem, which seeks to find an MIS with the highest weight \cite{10.1145/3618260.3649791}. Clearly, the MWIS problem is also \textbf{NP}-hard \cite{andrist2023hardness}, indicating that no polynomial time algorithm can solve MWIS unless $\textbf{P}=\textbf{NP}$. The MWIS problem has been applied in diverse fields \cite{banerjee2024error,duran2024survey,PMID:38718835,E220135,pattanayak2024maximal,E220448}. As a specific example, we examine the data routing and scheduling problem within vehicular ad hoc networks (VANETs) \cite{10816263}. This problem aims to optimize communication between vehicles and roadside units (RSUs) and among the cars themselves \cite{liu2015cooperative}. It involves the scheduling and routing of data transmissions to minimize conflicts and maximize resource utilization. We represent tentative schedules (TSs) as potential data transmissions, which can serve as a pending request via either infrastructure-to-vehicle or vehicle-to-vehicle communication. A graph can be constructed in the following manner: Each vertex represents a TS defined by a sender, a receiver, and the transmission mode, where the vertex weights indicate the throughput of data; the edges represent conflicts between TSs. Conflicts may arise, for example, when a vehicle serves as both a sender and a receiver in multiple TSs simultaneously, when RSUs are limited to transmitting only one data item at a time, or when a receiver cannot accept data from various senders concurrently. Consequently, identifying an MWIS in the graph yields an optimal scheduling solution that enhances transmission efficiency, reduces conflicts, and improves the overall performance of ITSs.

{Traffic signal control and coordination within ITSs can also be addressed using the MWIS approach \cite{9847014}. Initially, the intersection signal control problem is modeled as a graph (called phase conflict graph), wherein each signal phase is represented as a node. An edge is established between two nodes if the corresponding phases cannot be activated simultaneously due to crossing conflicts in their controlled traffic flows, indicating their conflict relationship. Furthermore, to assess the significance of each signal phase, weights are assigned based on factors such as traffic flow, reduction in average delay time, vehicle queue length, and the throughput capacity of the phase. After constructing the phase conflict graph, the objective is to identify an MWIS, that is, a set of non-conflicting signal phases that maximizes the total weight. This selection ensures that during each signal cycle, the chosen combination of phases facilitates the maximum number of vehicles passing smoothly through the intersection, thereby enhancing overall traffic efficiency.
}

% {\color{red}The Maximum Weighted Independent Set (MWIS) model can also be applied to specific parking optimization scenarios, such as optimal parking spot recommendation and shared parking scheduling in smart parking systems. In this modeling approach, parking spaces can be abstracted as vertices in a graph, with weights assigned based on factors such as spatial layout, historical usage frequency, and user preferences. For certain parking selection problems, such as avoiding excessive traffic concentration in specific areas or ensuring a more spatially dispersed allocation of parking spots, MWIS can be used to identify an optimal subset of parking spaces. In practical applications, by constructing a weighted parking graph and applying heuristic algorithms to solve the MWIS problem, the system can recommend better parking solutions and adjust parking scheduling in real time to reduce local congestion and improve parking space utilization.}

% While MWIS is commonly applied across various fields, developing efficient MWIS algorithms presents a significant challenge. 
{
The various applications of the MWIS problem in intelligent transportation systems inspire us to pursue efficient MWIS solvers. Despite the widespread use of MWIS across diverse fields, creating efficient algorithms for it continues to be a considerable challenge.} Traditional branch-and-bound exact algorithms perform satisfactorily for small-scale graphs, but their effectiveness diminishes as the graph size increases. To address this issue, researchers have applied specific data reduction rules to reduce the search space for large-scale graphs, simplifying the problem and reducing computational complexity \cite{gellner2021boosting,zheng2020efficient,xiao2021efficient,gu2021towards,grossmann2023finding}. However, these rules are not universally applicable; they are effective only for sparse graphs or graphs with specific properties and have limited applicability for general or dense graphs. Compared to the high computational complexity of exact algorithms, heuristic algorithms utilize efficient search strategies to identify near-optimal solutions rapidly \cite{zhu2024dual}. These algorithms effectively manage large-scale and complex traffic optimization scenarios \cite{li2022network,sudheera2022real}.  In this context, we propose an MWIS algorithm based on heuristic strategies, hoping to help address various optimization challenges within ITSs. 

Previous studies show that existing heuristic MWIS algorithms often suffer from slow convergence and suboptimal solutions \cite{nogueira2018hybrid,grossmann2023finding}.  
{For instance, HILS is recognized as an outstanding local search algorithm for MWIS \cite{nogueira2018hybrid}, yet it still depends on the random selection of vertices to formulate the initial solution. This approach can result in a subpar initial solution, potentially prolonging the convergence time. Furthermore, its perturbation strategy is single-patterned, making it challenging to ensure the effectiveness of the perturbation process. The HtWIS algorithm, on the other hand, prioritizes time efficiency at the expense of solution quality, failing to continuously optimize the solution \cite{gu2021towards}. Structure-based algorithms utilize numerous reduction rules during their execution \cite{gellner2021boosting}, which may lead to memory overflow issues. Memetic algorithms \cite{grossmann2023finding}, which are based on evolutionary processes, incorporate graph partitioning techniques, and require the maintenance of multiple solutions concurrently. The operations of mutation, selection, and crossover can further extend the convergence time of the solutions.
} 
% For this, we present three innovative strategies within the local search framework for the MWIS problem. 

{
For this, we present a novel local search algorithm for the MWIS problem.
Initially, the algorithm employs data reduction techniques, limiting their execution to a brief period. It then constructs the initial solution using two distinct methods based on the classification of the input graph, thereby enhancing the quality of the initial solution compared to previous local search approaches. Furthermore, the optimization process integrates three key innovations.} 
The first is a scores-based adaptive vertex perturbation (SAVP), which employs four scoring functions to refine the pattern of vertex insertion. {This diversifies the perturbation process, preventing the search from becoming stuck in local optima.} The second is a region location mechanism (RLM) that adjusts the search area when the optimization process stagnates. {This makes our algorithm significantly different from previous ones. When global region search becomes challenging, our algorithm shifts its focus to exploring areas that have been less frequently examined, thereby uncovering potential enhancements.} The third is a novel variable neighborhood descent framework (ComLS) that integrates various vertex exchange strategies. {This framework is derived from ILS-VND \cite{nogueira2018hybrid} and has been further developed with more complex vertex exchange strategies, enabling the generation of a greater variety of neighborhood structures. A reward mechanism is employed to efficiently guide the use of these vertex exchange strategies, further enhancing the quality of solutions.} These advancements generate a new local search algorithm, DynLS, tailored to solve the MWIS problem. This method leverages both global and local characteristics, enabling it to adaptively modify search patterns and enhance its capability to escape local optima. Experimental results indicate that DynLS outperforms leading MWIS algorithms in accuracy while achieving a fast convergence rate.

The remainder of this paper is structured as follows. In Section \ref {sec-2}, we provide a comprehensive overview of related work. Section \ref {sec-3} encompasses preliminary knowledge. The primary components of our algorithm, developed through a rigorous research process, are elaborated in Section \ref {sec-4}. Our experimental results, obtained through meticulous experimentation, are presented in Section \ref {sec-5}, and lastly, Section \ref {sec-con} offers the conclusion.

% \section{Materials and Methods}
\section{Related work}\label{sec-2}
Lamm et al. \cite{lamm2019exactly} introduced a branch-and-reduce algorithm for the MWIS problem. This approach employs reduction rules before branching, effectively minimizing the size of the graphs. Additionally, it utilizes local search algorithms to establish lower bounds for the solutions, significantly accelerating the search process. Zheng et al. \cite{zheng2020efficient} developed ExactMWIS, which recursively applies reduction rules to explore smaller graphs, successfully balancing efficiency with accuracy. Xiao et al. \cite{xiao2021efficient} formulated comprehensive reduction rules for MWIS, employing greedy algorithms for calculating lower bounds and weighted clique-covering methods to determine upper bounds. Gellner et al. \cite{gellner2021boosting} proposed a generalization of data reduction and transformation rules for MWIS, achieving effective solutions for many challenging instances but requiring considerable memory resources. Despite many strategies designed to enhance the efficiency of exact algorithms, a notable need remains for support in managing large-scale instances. Consequently, heuristic approaches \cite{ZHANG2024121567} are developed to yield high-quality solutions within acceptable time and space constraints. This paper focuses on designing heuristic algorithms for the MWIS problem and highlights relevant works in this domain.

Among the various heuristics, the greedy algorithm is noted for its simplicity, speed, and intuitive nature \cite{krivelevich2024greedy}. However, its principal drawback is the lack of consistency in producing optimal solutions. Consequently, it is often used with other methods, such as the local search algorithm \cite{zhang2023tivc}. The local search algorithm typically initiates with a starting solution and refines iteratively through insertion, deletion, or swapping operations. Empirical evidence indicates that local search algorithms frequently yield high-quality solutions in a relatively short amount of time. Andrade et al. \cite{andrade2012fast} pioneered iterated local search (ILS) to the MIS problem. They incorporated two vertex exchange strategies: the $(1,2)$-swap and the $(2,3)$-swap, where the $(i,j)$-swap entails removing $i$ vertices from the solution and adding $j$ vertices to it. Their research found that the $(1,2)$-swap was more efficient to implement compared to the $(2,3)$-swap. The $(2,3)$-swap was employed only when the $(1,2)$-swap could not be executed. The ILS approach also uses a perturbation technique to avoid the algorithm becoming trapped in local optima, that is, forcibly inserting $k$ non-solution vertices into the solution when search efficiency declines. In 2018, Nogueira et al. \cite{nogueira2018hybrid} applied ILS to address the MWIS problem. They introduced two types of swap: $(1,2)$-swap and $(\omega,1)$-swap while proposing a hybrid iterated local search (HILS) algorithm for MWIS. The weighted $(1,2)$-swap involves exchanging elements to enhance the solution's overall weight. Conversely, $(\omega,1)$-swap entails adding a vertex to the solution and removing its $\omega$ neighbors under the condition that the weight of the added vertex surpasses the combined weights of its $\omega$ neighbors. The HILS algorithm employs variable neighborhood descent (VND) \cite{CAVERO2022108680} to optimize these two vertex exchange strategies. While the HILS algorithm demonstrates robust performance in solving MWIS, it is characterized by slow convergence. 

In recent years, many researchers have combined local search with data reduction techniques to improve the performance of the ILS algorithm for MWIS on large graphs. In 2020, Zheng et al. \cite{zheng2020efficient} introduced two greedy algorithms, DtSingle and DtTwo. DtSingle utilizes a single vertex reduction technique, while DtTwo combines single vertex reduction and two vertex reduction techniques, specifically focusing on low-degree reductions. In 2021, Gu et al. \cite{gu2021towards} proposed the reducing and tie-breaking framework for  MWIS, comprising two primary phases: the tie-breaking and reduction phases. Although the tie-breaking strategy is heuristic and does not guarantee optimal solutions, this method still exhibits strong efficiency. In 2022, Dong et al. \cite{dong2022local} introduced a metaheuristic algorithm for MWIS, named METAMIS, which is based on the greedy randomized adaptive search procedure (GRASP). They introduced an AAP-move to assess if the solution could be improved and used a relaxed solution to guide local search. They showed that their algorithm works well with certain datasets, but it is uncertain whether it will perform well in all cases. In 2023, Gro$\beta$mann et al. \cite{grossmann2023finding} developed a memetic algorithm for MWIS by combining an evolutionary algorithm initially designed for MIS with local search strategies. This algorithm significantly emphasizes advanced data reduction techniques and investigates how various sequences influence the solution's accuracy and efficiency. However, excessive data reduction leads to significant time overhead, and the algorithm may exhibit suboptimal performance on certain dense graphs.  Nevertheless, it is still regarded as one of the most effective algorithms for MWIS.

The algorithms discussed previously typically perform well in either accuracy or efficiency but rarely in both. The algorithm proposed in this paper seeks to find a balance between the quality of the solution and computational efficiency, aiming to provide a high-quality solution within a reasonable time.

\section{Preliminaries}\label{sec-3}
All graphs under consideration are finite, simple, undirected, and weighted. A \textit{weighted graph} \( G \), with a vertex set \( V \) and edge set \( E \), can be expressed as a triple \( G = (V, E, \sigma) \), where \( \sigma \) is a weight function of \( G \), assigning each vertex \( v \in V \) a positive real number. The value \( \sigma_G(v) \) denotes the \textit{weight} of the vertex \( v \). For a subset \( S \subseteq V \), its \textit{weight} is represented as \( \sigma_G(S) \) and is defined as \( \sigma_G(S) = \sum_{v \in S} \sigma_G(v) \). For clarity, we utilize \( V(G) \), \( E(G) \), and \( \sigma_G \) to denote the vertex set, edge set, and weight function of a given graph \( G \), respectively. 
For two vertices $x, y\in V(G)$ in a graph $G$, we refer to the length of the shortest path between $x$ and $y$ as the \textit{distance} between $x$ and $y$ in $G$, denoted by $d_G(x,y)$. The longest distance between two vertices in $G$ is called the \textit{diameter} of $G$, denoted by $dia(G)$. If $d_G(x,y)=1$, then there is an edge (denoted by $xy$) connecting them, and $x$ and $y$ are said to be \textit{adjacent} each other. Generally, if  $d_G(x,y)=i (\geq 1)$, then one is called an \textit{$i$th-level neighbor} of the other in $G$. A first-level neighbor of a vertex is simply called a \textit{neighbor} of the vertex. The set of all $i$th-level neighbors of a vertex $v \in V(G)$ is called the \textit{$i$th-level neighborhood} of $v$ in $G$, denoted by $N_G^i(v)$,  i.e., $N_G^i(v)=\{u\mid u\in V(G) ~\text{and}~ d_G(u,v)=i\}$, and $N_G^i[v]= N_G^i(v)\cup \{v\}$ is called the \textit{closed ith-level neighborhood} of $v$ in $G$. $N_G^1(v)$ and $N_G^1[v]$ is often represented by $N_G(v)$ and $N_G[v]$, respectively. The \textit{$i$th-level neighborhood} of $S$ in $G$, denoted by $N_G^i(S)$, is the set of vertices in $V(G)\setminus S$ that has distance $i$ with a vertex in $S$, i.e., $N_G^i(S)=\{u  \mid u\in V(G)\setminus S ~\text{and}~ d_G(u,v)=i ~\text{for some} ~v \in S\}$, and $N_G^i(S)=N_G^i(S)\cup S$ is called  the  \textit{closed $i$th-level neighborhood} of $S$ in $G$. Analogously, $N_G^1(S)$ and  $N_G^1[S]$ is represented by $N_G(S)$ and $N_G[S]$, respectively. For a vertex $v\in V(G)$, the \textit{degree} of $v$ in $G$, denoted by $d_G(v)$, is the number of neighbors of $v$ in $G$, i.e.,  $d_G(v)=|N_G(v)|$. The \textit{maximum degree} of $G$, denoted by $\Delta(G)$ is defined as \(\Delta(G) = \max \{d_G(v) | v \in V(G)\}\). The \textit{average degree} of a graph $G$, denoted by $\overline{d}(G)$ is the average of degrees among all vertices of $G$, i.e., $\overline{d}(G)=\frac{\sum_{v\in V(G)} d_G(v)}{|V(G)|}$. We use $G-S$ to denote the resulting graph from $G$ by deleting all vertices in $S$ and their incident edges. When $S$ contains only one vertex $v$, we replace $G-\{v\}$ by $G-v$ simply.  The \textit{subgraph of $G$ induced by $S$}, or \textit{induced subgraph} simply, denoted as $G[S]$, is defined as $G[S]=G-(V(G)\setminus S)$.

\begin{definition}\label{def-IS}
Let \( G \) be a graph. A subset \( I \subseteq V(G) \) is referred to as an \textit{independent set} (IS) if the induced subgraph \( G[I] \) contains no edges. An IS is termed \textit{maximum} if no larger IS exists. The \textit{independence number} of \( G \), denoted by \( \alpha(G) \), represents the size of a maximum independent set (MIS).
\end{definition}

\begin{definition}\label{def-WIS}
Let $I$ be an IS of a graph $G$. If $\sigma_G(I)\geq \sigma_G(I')$ for any other IS $I'$ of $G$, then $I$ is called a \textit{maximum weighted independent set} (MWIS)  of $G$.  The \textit{weighted independence number} of $G$, denoted by $\alpha_{\omega}(G)$, is the weight of an MWIS.
\end{definition}

\begin{problem} \label{pro-WIS}
\normalfont{[\textbf{The MWIS problem}]} Given a graph $G$, the MWIS problem is to find an MWIS of $G$. 
\end{problem}

Given an independent set \( I \) in a graph \( G \), the \textit{tightness} of a vertex \( v \in V(G) \setminus I \), denoted as \( \tau_G(v) \), is the number of \( v \)'s neighbors that are included in \( I \), i.e., \( \tau_G(v) = |\{u \mid u \in N_G(v) \text{ and } u \in I\}| \). If the set \( I \cup \{v\} \) also forms an independent set in \( G \), then \( v \) is referred to as a \textit{free vertex related to $I$}, or simply a \textit{free vertex}. In this scenario, \( \tau_G(v) = 0 \).

\section{The DynLS Algorithm}
\label{sec-4}
%This section outlines the DynLS algorithm for solving the MWIS problem. 
DynLS maintains two types of solutions: the global best solution (denoted as \(BS\)) and the current solution (denoted as \(CS\)). {Figure \ref{fig1:flowchart} illustrates the overall process of DynLS. Initially, DynLS takes an undirected weighted graph \( G \) as input and applies reduction rules \cite{gellner2021boosting} to reduce the original graph \( G \) into a small kernel graph. Following this step, it utilizes the CIS algorithm to generate an initial solution. The process then progresses into a dynamic local search phase, where the solution (\( CS \)) is continuously refined until the time limit is reached. Ultimately, the algorithm returns the best solution \( BS \). Specifically, DynLS adjusts its execution process based on the parameter \( r_G \). If \( r_G \leq 2 \), DynLS implements the SimComLS algorithm until the time limit is reached, returning the best solution. Conversely, when \( r_G > 2 \), DynLS first operates the SAdpLS algorithm on graph \( G \) and subsequently on the \textit{local graph}. If this approach successfully improves the solution in the \textit{local graph}, DynLS repeats the SAdpLS on \( G \). If not, it resorts to the ComLS algorithm on \( G \) to further enhance the solution. Finally, when the time limit is reached, the best solution \( BS \) is returned.
}

\subsection{Initial solution construction}
We propose an algorithm named CIS to generate a high-quality initial solution. This algorithm constructs an IS by combining two methods: the reducing and tie-breaking framework (RTBF) \cite{gu2021towards} and a greedy-based strategy. Specifically, we use a parameter $r_G$ to determine which method to apply, where  $r_G= \min\{ \ell \mid \sum_{i=0}^\ell {(\overline{d}(G)})^i \geq \frac{1}{10} \times |V(G)| \}$. When $r_G>2$, the graph $G$ is relatively sparse, and the RTBF is employed to generate a solution. When $r_G\leq 2$, the graph $G$ is relatively dense, and a greedy strategy based on a scoring function $init_s(v) = \frac{\sigma_G(v)}{(d_{G}(v))^{0.5}}$ is used to generate a solution. This involves iteratively selecting a vertex with the highest score, adding it to the solution, and removing it and its neighbors from $V(G)$ until $V(G)$ is empty. For a detailed description of the construction process of the CIS, please see Algorithm S1 in Section S1 of the supplement file.

\begin{figure}[h]
\centerline{\includegraphics[width=0.8\linewidth]{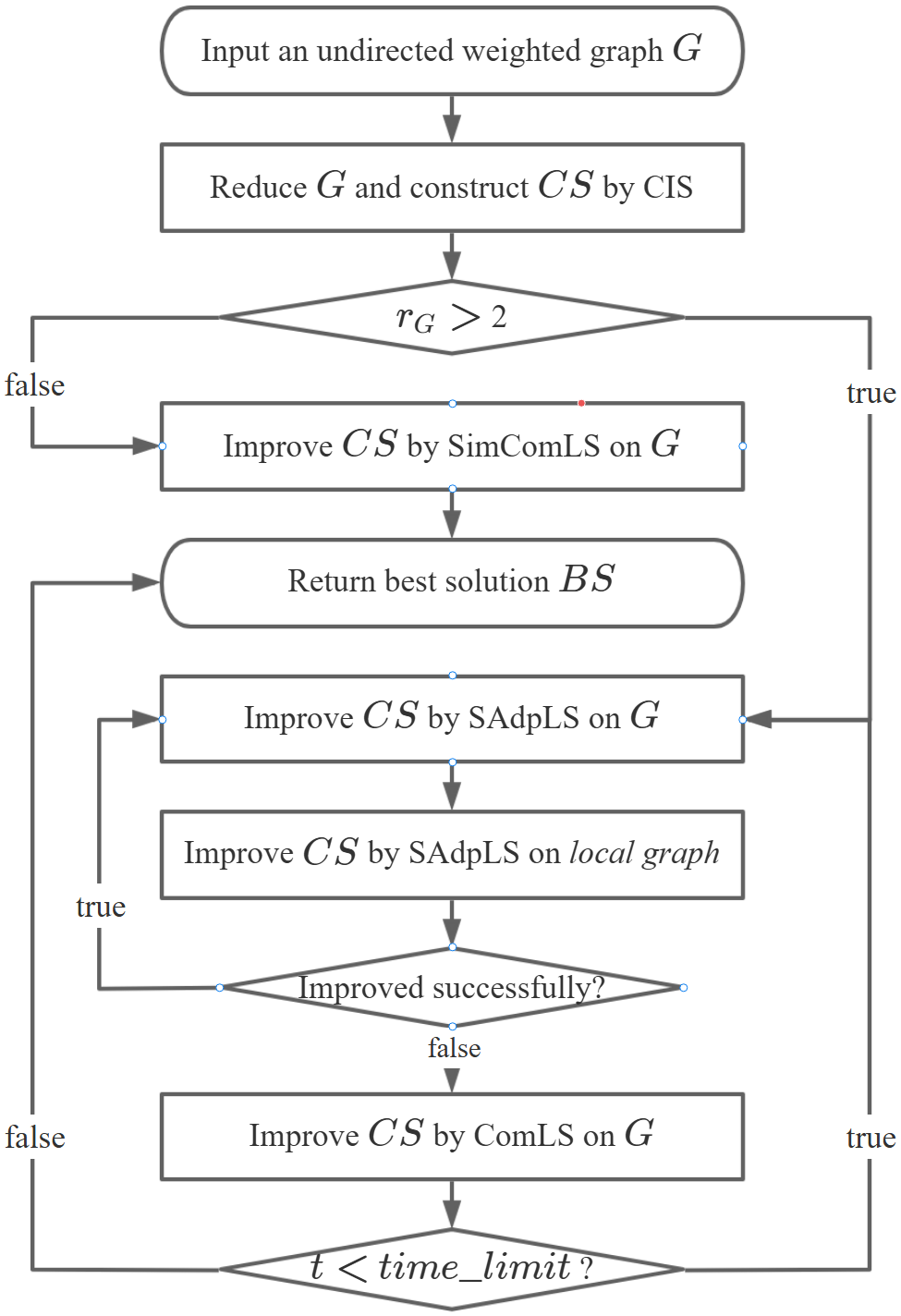}}
    \caption{The flowchart of the DynLS algorithm.}
    \label{fig1:flowchart}
\end{figure}

\subsection{Local search}
This section presents the local search process for DynLS, which is developed based on the iterated local search framework \cite{andrade2012fast}. Our local search algorithm encompass three enhanced approaches: A scores-based adaptive vertex perturbation strategy, a region location mechanism, and a variable neighborhood descent framework that integrates a series of vertex exchange strategies with a reward mechanism.

\subsubsection{Scores-based adaptive vertex perturbation}
In the ILS framework for MWIS, a commonly utilized perturbation strategy involves adding \( k \) random non-solution vertices to the current solution while removing their neighbors \cite{nogueira2018hybrid}. Here, \( k \) serves as a hyperparameter. However, relying on a constant value for \( k \), especially in sparse graphs, may restrict the method's adaptability to varying search scenarios. In addition, selecting vertices randomly in a global manner at each iteration can extend the time required for solution convergence. To overcome these challenges, we enhance the number of insertion vertices and the criteria for their selection. Our approach employs two variables: \( iter \), which tracks the total number of iterations, and \( uiter \), which counts the number of iterations without any improvement to the solution. The total number of insertion vertices, denoted as \( num \), consists of two components: a base number \( base\_num \) and a probabilistic component \( pro\_num \). Specifically, \( num = base\_num + pro\_num \), where \( pro\_num \) is defined as \( i + 1 \), with a probability that is proportional to \( \frac{1}{2^i} \) for \( i \geq 1 \). When \( uiter \) indicates that the solution has not improved for a prolonged period, we increase the base number to intensify the perturbation.

For a given vertex \( v \), the term ``visited'' is applied when the vertex is either added to or removed from the current solution \( CS \). The function \( age(v) \) represents the ``age'' of \( v \), indicating the number of iterations since its last state change. Each time the vertex is visited, \( age(v) \) is reset to 0 and then incremented with each following iteration. The ``age'' of vertices has been identified as an effective strategy for selecting high-quality vertices for MIS \cite{andrade2012fast}. However, its application can be problematic when the criteria for vertex selection are insufficient. We propose three new selection criteria to complement the ``age'' strategy, including  \( freq(v) \),  \( change(v) \), and \( loss(v) \). This integrated approach expands the range of options for vertex selection and helps mitigate conventional perturbation strategies' limitations. Specifically, \( freq(v) \) represents the frequency of visiting \( v \), which increases by one each time the vertex is visited. The function \( change(v) \) monitors the number of times the vertex \( v \) has been visited when \( CS \) has improved and reflects changes to the state of \( CS \). When \( v \) is added to \( CS \), the count for \( change(v) \) increases by one. Conversely, when \( v \) is removed from \( CS \), the count decreases by one. \( loss(v) \) quantifies the loss incurred when adding \( v \) to \( CS \) and removing its neighbors, defined as \( loss(v) = \sigma_{G}(N_G(v) \cap CS) - \sigma_{G}(v) \). A negative value for \( loss(v) \) signifies a gain in weight.

Initially, the values of \(freq(v)\), \(age(v)\), and \(change(v)\) for every vertex \(v\) are set to 0, while \(loss(v)\) is calculated based on \(CS\). The strategy pool comprises strategies derived from \(freq(v)\), \(age(v)\), \(change(v)\), and \(loss(v)\). Throughout each round of searching for solutions, a strategy is randomly selected from the current pool, denoted as \(op\), and the value of this selected strategy for each vertex is called the vertex's \(score\). The ``best'' vertex according to each of these four criteria is defined as follows:

\begin{itemize}
\item  For \(freq(v)\), a lower score is preferred. A smaller \(freq(v)\) indicates that vertex \(v\) has been visited infrequently, which promotes a more balanced exploration of the entire space and reduces repetitive searches in local areas.

\item In the case of \(age(v)\), a higher score is desirable. An increased \(age(v)\) for a non-solution vertex implies that it has been excluded from the solution for an extended duration. Including such vertices could effectively disrupt the current solution.

\item Regarding \(change(v)\), a larger score is favored. A higher \(change(v)\) suggests that vertex \(v\) has frequently been part of the solution. Incorporating the vertex outside the solution with the highest \(change\) may facilitate new enhancements when the current solution is stagnant.

\item  Finally, for \(loss(v)\), a smaller score is preferred. Non-solution vertices with weights exceeding the total weight of their solution neighbors should be considered for inclusion in the solution. 
\end{itemize}

The preceding discussion introduces a local search procedure named $\mathrm{SAdpLS}$, designed to enhance an existing solution. For a detailed overview of the pseudocode and a comprehensive explanation of the $\mathrm{SAdpLS}$ algorithm, please refer to Section S2 of the supplementary file.

\subsubsection{Region location mechanism}
When a local search procedure fails to enhance the current best solution within a designated area, broadening the search to other regions becomes essential. Exploring less frequently explored areas can yield new improvements and improve the overall quality of the solutions. For this, we introduce an adaptive region location strategy that enables targeted searches in rarely frequent regions, ensuring a more comprehensive exploration of the entire graph.

For a graph \( G \) and a set of reduction rules, the resulting graph, denoted as \( kG \), derived from applying these reduction rules to \( G \) is referred to as a \textit{kernel} of \( G \). In the following sections, we introduce an innovative concept called the \textit{local graph}.

\begin{definition}
Let $G$ be a graph and $CS$ an IS of $G$. For a vertex $v\in V(G)$ and an integer $r<dia(G)$, let 
$D= \cup_{i=1}^r N_G^i[v] \setminus \{u \mid u \in N_G^r(v) \land \exists w \in CS \land wu\in E(G) \land w \in N_G^{r+1}(v)\}$. We call the subgraph of $G$ induced by $D$ a \textit{$(v, r, CS)$-local graph associated with $v$}, where  $v$ is the \textit{center}  and $r$ is called the \textit{radius} of the local graph. 
\end{definition}

The local graph is a valuable tool for identifying the optimal region for a targeted search to discover the best solution. The primary objective is to identify a high-quality center. We select a vertex $v$ with the smallest $freq$ value as the center to achieve this. The rationale behind this approach is as follows: for a vertex $v$, a low $freq(v)$ suggests that $N_G^2(v)$ contains vertices with low $freq$ values, indicating that the local graph centered at $v$ may have been explored less frequently.

\subsubsection{Vertex exchange strategy}
The 2-improvement heuristic \cite{andrade2012fast} involves replacing one vertex $v$ in the current solution $CS$ with two vertices in $N_G(v)$. Therefore, 2-improvement considers only the first-level neighborhood of a vertex in $CS$. This limited search space may hinder solution improvement. We introduce a new vertex exchange strategy, called $(v;x,y)$-exchange, which incorporates second-level neighborhoods in the exchanging process and expands the scope of vertex exchange.

\begin{definition}
Let $CS$ be an independent set of a graph $G$. For each vertex $v\in CS$, a $(x, y)$-exchange related to $v$, denoted by $(v; x, y)$-exchange, refers to finding an independent set $S \subseteq N_G(v)$ such that $S$ contains exactly $x$ vertices with tightness one and $y$ vertices with tightness two, adding the vertices in $S$ to $CS$, and deleting the vertices in $N_G(S) \cap CS$.
\end{definition}

Given a vertex \( v \), a \( (v; x, y) \)-exchange may not exist if the neighborhood \( N_G(v) \) contains no subset \( S \) that satisfies the conditions specified in the definition. Moreover, even when a \( (v; x, y) \)-exchange is possible, we also refrain from executing the vertex exchange operation if adding \( S \) to the current solution \( CS \) results in a solution with a lower weight than that of \( CS \). In our DynLS algorithm, to enhance efficiency, we prioritize \( (x, y) \)-exchanges with smaller values for \( x \) and \( y \). This includes utilizing \( (x, 0) \)-exchanges for \( x \leq \Delta(G) \) and \( (x, y) \)-exchanges for any \( x \in \{1, 2, 3\} \) and \( y \in \{1, 2\} \).

To enhance search efficiency, we integrate \((x,y)\)-exchanges with \((\omega,1)\)-swap \cite{nogueira2018hybrid}, 2-improvement, and \((2, 3)\)-swap \cite{andrade2012fast}, resulting in a series of exchange modules. Utilizing the variable neighborhood descent (VND) approach \cite{brimberg2023variable}, we search for solutions with these modules. The exchange modules can be classified into three distinct types. The first type, referred to as \(EM\), comprises two neighborhood structures associated with a \((\omega,1)\)-swap and a \((x,y)\)-exchange. The \((\omega,1)\)-swap is recognized for its high execution efficiency, while the \((x,y)\)-exchange can significantly enhance solutions. The second type, denoted as \(A\), consists of three neighborhood structures: a \((\omega,1)\)-swap, a 2-improvement, and a \((1,1)\)-exchange. The third type, labeled as \(B\), includes two neighborhood structures corresponding to a \((2,3)\)-swap and a \((x,0)\)-exchange. The search process is described as follows. It starts by identifying improvements within the first neighborhood structure. If an improvement is achieved, the search proceeds within this structure. When no further improvements can be identified, the search shifts to the second neighborhood structure. The search returns to the first neighborhood structure if an improvement is found. This alternating process continues until no further improvements can be made in the second structure. The third neighborhood structure is treated similarly for the \(A\) type exchange module. Upon completing the search within each neighborhood structure, free vertices are added to the solution in descending order of their weights until the solution is maximized.

The $A$ type exchange module is straightforward and is employed in every iteration round. Following its application, we select an appropriate module from the set of exchange modules, denoted as $EM$, to substantially enhance the current solution. The $B$ type exchange module is utilized only in specific instances to improve the algorithm’s generalization capability. Given that different exchange modules in $EM$ involve varying numbers and types of vertices, specific models are better suited for particular scenarios. To determine the most suitable strategy from $EM$, we implement a reward mechanism represented as $Re$. Each exchange module $a \in EM$ is assigned an initial reward value of $Re(a) = 1$. During each iteration round, the algorithm selects an exchange module $a$ from $EM$ for the subsequent local search using the roulette wheel selection algorithm, based on the module selection function defined as:
\[
\pi(a) = \frac{Re(a)}{\sum_{a \in EM} Re(a)}
\]
After completing the vertex exchange process, we update the value of $Re(a)$ following the approach described in \cite{nogueira2018hybrid}, distinguishing between positive and negative rewards. Positive rewards motivate exchange modules with higher reward values, while negative rewards diminish the likelihood of using modules with lower reward values. By implementing this strategy, our algorithm can leverage different exchange modules at various stages of the search process, thereby enhancing overall efficiency. To facilitate periodic termination of the search process, we introduce an integer variable, \( sum\_Re \), which is based on the value of \( Re \). The initial value of \( sum\_Re \) is set to the total number of exchange modules in \( EM \). If a module \( a \in EM \) results in an increase of \( Re(a) \) by 1, 2, or 3, then \( sum\_Re \) is correspondingly incremented by 1, 2, or 3. Conversely, if no such module is identified, \( sum\_Re \) is decreased by 1. The search process concludes when \( sum\_Re \) is reduced to its initial value.

Utilizing the exchange modules and the reward mechanism, a complex local search algorithm for the MWIS problem is developed and designated as ComLS. The pseudo-code and a comprehensive algorithm description are presented in Section S3 of the supplement file.

\vspace{-0.2cm}
\subsection{The DynLS algorithm}
This section offers a comprehensive description of our DynLS algorithm. As depicted in Algorithm \ref{alg-main}, the algorithm takes a graph \(G\) and a time limit \(time\_limit\) as input (line 1) and produces the best solution \(BS\) (line 2). The algorithm begins by reducing the graph \(G\) into a kernel (also referred to as \(G\)) through the application of reduction rules outlined in \cite{gellner2021boosting} (line 3). It then constructs an initial solution using the CIS algorithm (line 4). During the initialization phase, the algorithm calculates an initial radius \(r_G\) necessary for constructing local graphs (line 5). If \(r_G \leq 2\), this indicates that the graph \(G\) is relatively dense. The SimComLS algorithm enhances the current solution \(CS\) (lines 6-7) in such instances. Conversely, if \(r_G > 2\), the algorithm enters an iterative process that involves three key stages before reaching the time limit. First, it accelerates convergence speed using SAdpLS on global graphs (line 10). Next, it enhances the capability to escape from locally optimal solutions by focusing on low-frequency areas, i.e., the local graphs (lines 11-25). Finally, it refines the solution quality using ComLS on global graphs (lines 26-28).

The algorithm begins each iteration by enhancing the current solution using SAdpLS (line 10). Suppose SAdpLS becomes trapped in a local optimum, indicating difficulty in improving the solution on the global graph. In that case, the algorithm shifts its focus to local graphs to escape this local optimum and further refine the solution (lines 12-25). Our approach is exclusively centered on local graphs derived from the set \(CS\), resulting in a total of \(|CS|\) local graphs for consideration.
We partition these local graphs into 100 segments to enhance search efficiency, allowing the algorithm to select one segment for local search at a time. Before executing the local search on these segments, the algorithm calculates the number \(\ell\) for each segment of local graphs as \(0.01 \times |CS|\). It also initializes several parameters, including a counter \(c\), a boolean variable \(flag\), a set \(list\), and an array \(improve_{false}\) with a length of \(|V|\) (line 11). Following this, the algorithm enters a loop to conduct the local search on the local graphs (lines 12-25).

The process begins by selecting a vertex \( v \) with the smallest \( freq(v) \)  from $list$. This selection results in updates to the list and a variable \( c \) (line 13). A \((v, r', CS)\)-local graph \( LG \) is then constructed with center \( v \) and radius \( r' = r_G + improve_{false}(v) \), where \( improve_{false}(v) \) tracks the number of times the current best solution has not been enhanced when exploring local graphs centered at \( v \) (line 14). Let \( solu_1 \) denote the intersection of the vertex set of \( LG \) and the current solution \( CS \). An initial solution \( solu_2 \) for \( LG \) is generated using a greedy approach based on vertex weights (line 15). The SAdpLS algorithm is subsequently applied to enhance \( solu_2 \). If \( solu_2 \) surpasses \( solu_1 \) in quality, the algorithm updates the current best solution \( BS \) as well as the current solution \( CS \) and sets \( flag \) to one, indicating that an improvement in the best solution has occurred during this stage (line 18). Conversely, if \( solu_2 \) does not exceed \( solu_1 \), the value of \( improve_{false}(v) \) is incremented by one to broaden the search area in the next round (line 20). If the current best solution \( BS \) fails to improve after searching \( \ell \) local graphs, the algorithm increases \( \ell \) to expand the search region (line 23). In situations where no improvements are found during the local search on local graphs (i.e., \( flag = 0 \)), the algorithm implements the complex local search procedure and updates the current solution \( CS \) (line 27).

\begin{algorithm}[h]
    \caption{DynLS algorithm}
    \scriptsize
    \label{alg-main}
    % \small
    \begin{algorithmic}[1]
    % \fontsize{6pt}{8pt}\selectfont % 设置行号大小
    \STATE \text{\textbf{Input:}} A $G = (V,E,\sigma)$, $time\_limit$
    \STATE \text{\textbf{Output:}} An independent set $BS$
    \STATE $G$ $\xleftarrow{}$ reduce $G$ by applying reduction rules;
    \STATE $r_G \xleftarrow{}  \min\{ \ell \mid \sum_{i=0}^\ell {(\overline{d}(G)})^i \geq \frac{1}{10} \times |V(G)| \}$;
    \STATE $CS$ $\xleftarrow{}$  CIS($G$, $r_G$), $BS$ $\xleftarrow{}$  $CS$;
    \IF{$r_G \leq 2$}
        \STATE $BS$ $\xleftarrow{}$ \  SimComLS($G$, $CS$, $BS$, $time\_limit$), $CS$ $\xleftarrow{}$ \  $BS$;
    \ELSE
        \WHILE {$t \leq time\_limit$}
            \STATE $BS$ $\xleftarrow{}$ \ SAdpLS($G$, $CS$, $BS$, $-1$), $CS \xleftarrow{} \  BS$;
            \STATE $\ell$ $\xleftarrow{}$ \  $\frac{|CS|}{100}$, $c$ $\xleftarrow{}$ \  0, $flag$ $\xleftarrow{}$ \  0, $list \xleftarrow{} \ CS$, $improve_{false}(v) = 0$ for $v \in V(G)$;
            \WHILE {($c \leq min\{|CS|,  \ell\}$) $\land$ $list\neq \emptyset$}
                \STATE $v$ $\xleftarrow{}$ \  a vertex with smallest $freq(v)$ from $list$, $list \xleftarrow{} \  list \setminus \{v\}$, $c$ $\xleftarrow{}$ \ $c$ + 1;
                \STATE  $LG$ $\xleftarrow{}$ a $(v, r', CS)$-\textit{local graph} with center $v$ and radius $r'= r_G + improve_{false}(v)$;
                \STATE $solu_1$ $\xleftarrow{}$ \  \{$v \mid v \in V(LG) \cap CS$\}, $solu_2 \leftarrow$ \  an initial solution of $LG$ obtained by greedy method;
                \STATE $solu_2$ $\xleftarrow{}$ \ SAdpLS($LG$, $solu_2$, $solu_1$,$searchDepth$);
                \IF {$\sigma_G(solu_2) > \sigma_G(solu_1)$} 
                    \STATE $BS \xleftarrow{} \  (BS \setminus solu_1) \cup solu_2$, $CS \xleftarrow{} \  (CS \setminus solu_1) \cup solu_2$, $flag $ $\xleftarrow{}$ \  1;
                \ELSE
                    \STATE $improve_{false}(v)$ $\xleftarrow{}$ \  $improve_{false}(v)$ + 1;
                \ENDIF
                \IF {($flag$ = 0)  $\land$ ($c$ = $l$)}
                    \STATE $\ell$ $\xleftarrow{}$ \ $\ell + \frac{|CS|}{100}$;
                \ENDIF
            \ENDWHILE
            \IF {$flag$ = 0}
                \STATE $BS \xleftarrow{} \ $ ComLS($G,CS,BS$), $CS \xleftarrow{} \  BS$;
            \ENDIF
        \ENDWHILE
    \ENDIF
     \RETURN $BS$;
    \end{algorithmic}
\end{algorithm}

\begin{theorem} \label{timecom}
For a graph with $n$ vertices and $m$ edges, the time complexity for the DynLS algorithm in one iteration is \(O(n \times \max\{n \log n, mn, \Delta^2\} \), where \(\Delta\) is the maximum degree of the graph.
\end{theorem}

\begin{proof}
The reduction rules and initial solution generation in the DynLS algorithm are executed only once, so their running times are not included in the analysis of a single iteration.

When \( r_G \leq 2 \), the algorithm utilizes the \( \text{SimComLS} \) algorithm, which runs in \( \max\{O(mn), O(n\Delta^2)\} \) (see Section S3 of the supplementary file). The algorithm enters a loop when \( r_G > 2 \). Each iteration begins with the \( \text{SAdpLS} \) algorithm, requiring \( O(mn) \) time (see Section S2 of the supplementary file). Updating parameters and performing comparisons take \( O(1) \) time. The algorithm then constructs the local graph \( LG \), marking and enumerating its vertices, which takes \( O(|V(LG)|) \) time, where \( |V(LG)| \leq n \). An initial solution for \( LG \) is obtained through a greedy approach that involves sorting the vertices, leading to a time complexity of \( O(|V(LG)| \log |V(LG)|) \). Given that \( |CS| < n \), the total time complexity for the loop is at most \( O(n \times \max \{n \log n, mn\}) \). Finally, the \( \text{ComLS} \) algorithm is executed with a time complexity of \( \max\{O(mn), O(n\Delta^2)\} \) (see Section S3 of the supplementary file). Therefore, the overall time complexity for a single iteration of DynLS is \( O(n \times \max\{n \log n, mn, \Delta^2\}) \). 
%~~~~~~~~~~~~~~~~~~~~~~~~~~~~~~~~~~~~~~~~~~~~~~~~~~~~~~~~~~~~~~~~~~~\(\blacksquare\)
\end{proof}

\section{Experiments}\label{sec-5}

\iffalse
In this section, we evaluate the performance of DynLS, comparing it to five state-of-the-art algorithms for solving the MWIS problem across various benchmark datasets. Section \ref{sec5-1} provides an overview of the experimental setup. Section \ref{sec5-2} explores the key parameters, \(M_1\) and \(M_2\). Section \ref{sec5-3} presents the final experimental results, focusing on the competing algorithms' solution quality and average convergence time across different instances, along with a detailed analysis and discussion of these findings. Lastly, Section \ref{sec5-4} describes the ablation experiments designed to illustrate the contributions of each component of the proposed algorithm.
\fi

\subsection{Experimental setup} \label{sec5-1}

\subsubsection{Parameters}
All experiments were conducted on CentOS Linux release 7.6.1810, utilizing an Intel(R) Xeon(R) Gold 6254 processor with 128 GB of RAM. The code, developed in C++, was compiled with GCC using the -O3 optimization option. The algorithm (requiring seed values) was executed five times for each instance, employing seeds 1 through 5, with a cutoff time set at 1000 seconds. We recorded the best solution \(max_w\) and the average solution \(avg_w\) across the five runs for each instance and algorithm. 
An “N/A” indicates a feasible solution could not be obtained within 1200 seconds. We implemented the reduction rules developed by Alexander Gellner et al. \cite{gellner2021boosting} for kernelization purposes. The reduction process was limited to 200 seconds, and we utilized the \(C_{\text{fast}}\) reduction configuration \cite{gellner2021boosting}. The parameters \(M_1\) in SAdpLS (Algorithm S2) and \(searchDepth\) in DynLS (Algorithm \ref{alg-main}) were both set to 100, while \(M_2\) in SAdpLS was fixed at 3000. The parameter $t$ of BMS is set to 50 as previous research \cite{zhang2023tivc}. The key parameters \(M_1\) and \(M_2\)  will be fine-tuned as detailed in Section \ref{sec5-2}.

\subsubsection{Algorithms and datasets}
We compare DynLS with five state-of-the-art algorithms that are specifically designed to address the MWIS problem: the HILS algorithm \cite{nogueira2018hybrid}, the HtWIS algorithm \cite{gu2021towards}, the Solve algorithm \cite{xiao2021efficient}, the Cyclic-Fast algorithm \cite{gellner2021boosting}, and the m$^2$wis + s algorithm \cite{grossmann2023finding}. 
{
HILS employs a local search framework that explores two neighborhood structures, the $(\omega, 1)$-swap and the $(1, 2)$-swap. It incorporates VND to the above two neighborhoods, resulting in a strong performance across various graph types.   
HtWIS is a fast approach that utilizes straightforward reduction rules degree-one, degree-two, and single-edge reduction rules and pruning techniques to approximate the MWIS.  
Solve features a robust array of reduction rules and functions as a standard branch-and-reduce exact algorithm.    
Cyclic-Fast adopts a branch-and-reduce strategy, applying struction methods to balance speed and solution quality. 
m$^2$wis + s is an evolutionary technique that combines diverse reduction rules and graph partitioning-based crossover methods with local search strategies, yielding commendable results for many large instances. However, as the code for the  m$^2$wis + s algorithm is not publicly available, we will only compare our results with those reported in their publication \cite{grossmann2023finding}.
}
% All algorithms, except m$^2$wis + s, are available online. 
% Please refer to Section S4 of the supplementary file for detailed descriptions of the algorithms.

The dataset comprises nine notable benchmarks: Dimacs \footnote{https://networkrepository.com/}, Facebook \footnote{https://networkrepository.com/}, Sparse \footnote{https:// sparse.tamu.edu}, Network \footnote{https://networkrepository.com/}, Fe, Mesh, Osm \footnote{https://www.openstreetmap.org}, Snap \footnote{https://snap.stanford.edu/data/}, and Ssmc \footnote{https:// sparse.tamu.edu}. Together, these benchmarks contain a total of 360 instances. 
Based on their methods of weight assignment, the datasets can be classified into two categories: Family A and Family B. 
Family A encompasses the Dimacs, Facebook, Sparse, and Network datasets. Following the method described in \cite{gu2021towards}, these datasets assign a weight to each vertex \(v\) using the formula \((id(v) - 1) \% 200 + 1\), where \(id(v)\) indicates the identifier of the vertex. 
{
(1) Dimacs dataset is renowned for investigating graph theory and combinatorial optimization problems. It comprises 70 synthetic and real-world instances with vertex counts ranging from 776 to 50912018 \cite{partitioning58810th}. (2) Facebook dataset, featuring 109 instances, is focused on social network analysis, primarily derived from the social connections among users on the Facebook platform. The vertex counts in this dataset range from 962 to 3097165 \cite{rossi2015network}. (3) Sparse dataset consists of sparse graphs widely employed in graph algorithms and machine learning, containing 45 instances with 1000 to 710145 vertices. (4) Network dataset encompasses graph data from various fields, including road network analysis, scientific research, social networks, technology, and the internet. It comprises 43 instances, with vertex counts ranging from 3031 to 23947347 \cite{rossi2015network}.
}
All datasets in Family A are available for online download.

Family B comprises the remaining five datasets, which utilize a random and uniform distribution to assign vertex weights between 1 and 200, as outlined in \cite{grossmann2023finding}. 
{
(1) Fe Dataset consists of seven 3D meshes with vertex counts ranging from 11143 to 143437. These meshes are generated using the finite element method \cite{soper2004combined}. Researchers frequently utilize these datasets to simulate and analyze physical phenomena across various fields, including engineering, physics, medicine, and fluid dynamics. (2) Mesh Dataset encompasses 15 dual graphs, with vertex counts between 4419 and 1087716, derived from well-known triangle meshes \cite{sander2008efficient}. Triangle meshes are foundational structures in computer graphics, computational geometry, and mesh processing, as they effectively represent the surfaces of three-dimensional objects. (3) Osm Dataset includes real-world graph data sourced from the OpenStreetMap (OSM) platform \cite{ barth2016temporal}. These graphs depict connections between cities, roads, and buildings, making them invaluable for geographic information systems, traffic analysis, route planning, and map visualization applications. It features 34 instances with vertex counts ranging from 1018 to 46221. (4) Snap Dataset consists of large social network datasets, containing 31 instances with vertex counts varying from 5242 to 4847571 \cite{jure2014snap}. It facilitates research in network science, social network analysis, graph theory, and related disciplines. (5) Ssmc Dataset incorporates data from diverse domains, including social networks, image processing, and geographic information systems \cite{kolodziej2019suitesparse}. It contains six instances with vertex counts ranging from 25181 to 710145.}
We thank Professor Darren Strash for his generous provision of the datasets in Family B.

% For additional details regarding the graphs within these datasets, please refer to Section S5 of the supplementary file.

\subsection{Tuning of parameter $M_1$ and parameter $M_2$} \label{sec5-2}

The parameters \( M_1 \) and \( M_2 \) in SAdpLS (Algorithm S2) are essential to the local search phase and greatly influence the efficiency of the region location mechanism.  We explored various combinations for \( M_1 \) from \{50, 100, 150, 200\} and for \( M_2 \) from \{2000, 3000, 4000\}, resulting in a total of 12 unique combinations. Each combination was tested using seeds 1, 2, 3, 4, and 5 across ten graphs of different sizes (see Table S1 of the supplementary file Section S4 for detailed information). Each run was conducted over 1000 seconds to identify the optimal parameter combination. The results include the best and average solutions for each graph across the five runs. Note that we assessed the algorithm’s performance without employing reduction rules, given that the tested parameters pertain solely to the local search phase, and the reduced graph would be too small to showcase any differences effectively.

The results are presented in Table S2 of the supplementary file in Section S4. 
{From the results, it can be observed that, when evaluating the \( max_w \) values, the combinations 100-3000, 150-4000, 200-3000, and 200-4000 achieved the highest values, each delivering the best solution for five instances. In terms of average solution quality, the 100-3000 combination secured the best solution for four instances, while the 150-4000, 200-3000 and 200-4000 combinations yielded the best solutions for zero, two and three instances, respectively. Thus, the 100-3000 combination emerges as the optimal choice. Moreover, \( M_1 \) denotes the depth of the \textit{local graph} search, and it is essential for the algorithm to adapt flexibly to various search scenarios. Consequently, a setting of 100 is deemed more suitable than 200. Therefore, the final parameter selection is 100-3000, with \(M_1\) set to 100 and \(M_2\) set to 3000.

%first considering the $max_w$ values, the combinations 100-3000, 200-3000, and 200-4000 ranked the highest in terms of the number of $max_w$, each achieving the best solution for 5 instances. When considering the average solution quality, the 100-3000 combination obtained the best solution for 4 instances, while the 200-3000 and 200-4000 combinations obtained the best solutions for 2 and 3 instances, respectively. Therefore, the 100-3000 combination is the optimal choice. Additionally, \( M_1 \) represents the depth of the \textit{local graph} search, and we also hope the algorithm can flexibly switch between different parts to adapt to various search scenarios. Thus, 100 is more suitable than 200. Therefore, the final parameter choice is 100-3000, i.e., \(M_1\) set to 100 and \(M_2\) set to 3000.

}
% We identified the optimal combination of parameters by selecting the one that yields the highest sum of the best and average results. As the table outlines, the ideal combination is \(M_1 = 100\) and \(M_2 = 3000\), which produces the highest sum of nine. Consequently, we finalized our parameter configuration to \(M_1 = 100\) and \(M_2 = 3000\).

\subsection{Experimental results} \label{sec5-3}
Due to space limitations, we have only included the experimental data for the Network, Fe, Snap, and Ssmc datasets in the main text; see Tables \ref{Experimental-results-on-network}-\ref{Experimental-results-on-ssmc}. Experimental results for the Dimacs, Facebook, Sparse, Mesh, and Osm datasets can be found in Tables S3-S7 in Section S5 of the supplementary file. Note that the m$^2$wis + s algorithm is unavailable; thus, it was compared only with the instances presented in its original paper \cite{grossmann2023finding}, specifically within Family B. Since the HtWIS and Solve algorithms do not require seed values, they were executed only once. The results of m$^2$wis + s, HtWIS, and Solve are represented as ``$value_w$''. We analyze the instances in which each algorithm achieves the maximum $max_w$, $value_w$, or $avg_w$ compared to the others. In each instance, the maximum value of $max_w$ or $value_w$ among all compared algorithms is presented in bold. DynLS, Cyclic-Fast, and HILS algorithms are evaluated based on their $max_w$ values, while Solve, HtWIS, and m$^2$wis + s are assessed according to their $value_w$. In addition, the average solution $avg_w$ for DynLS, Cyclic-Fast, and HILS is analyzed independently, with the maximum $avg_w$ among these three algorithms also highlighted in bold for each instance.

\subsubsection{Comparison of solution quality for Family A}
In this family of datasets, the DynLS algorithm was compared against four algorithms: HILS, HtWIS, Solve, and Cyclic-Fast. Across the datasets, DynLS consistently demonstrated significant outperformance compared to other algorithms. Especially, DynLS achieved the highest \(max_w\) and \(avg_w\) values for all instances in the Network benchmark, as shown in Table \ref{Experimental-results-on-network}. Next, we detail the performance results of these algorithms on the other datasets.

For the Dimacs dataset, DynLS emerged as the leading performer in identifying the highest \(max_w\) and \(avg_w\) values, as illustrated in Table S3 in Section S5 of the supplementary file. Among the 70 instances analyzed, DynLS achieved the highest \(max_w\) values for 68 instances, significantly surpassing the performance of HILS, HtWIS, Solve, and Cyclic-Fast, which attained the maximum \(max_w\) (or $value_w$) for only ten, four, 18, and 29 instances, respectively. DynLS was unable to find the highest \(max_w\) values for only two instances ``adaptive'' and ``fe-ocean.'' However, HILS, Solve, and Cyclic-Fast encountered difficulties in identifying feasible solutions for numerous instances, particularly large-scale instances like ``hugetrace-xxxxx'', ``hugetric-xxxxx'', and ``rgg\_n\_2\_2x\_s0.'' Moreover, DynLS continued to excel by securing the highest \(avg_w\) values for 69 instances, significantly outpacing HILS and Cyclic-Fast, which found the highest \(avg_w\) values on just eight and 25 instances, respectively.

Regarding the Facebook dataset, DynLS exhibited a notable advantage over its competitors, as illustrated in Table S4 of the supplementary file Section S5. Among the 109 instances analyzed, DynLS secured the highest \(max_w\) and \(avg_w\) values in 107 instances. In contrast, the second-best performer, HILS, achieved the highest \(max_w\) values in 11 instances and the highest \(avg_w\) values in four instances. Furthermore, the algorithms HtWIS, Solve, and Cyclic-Fast reached the highest \(max_w\) (or $value_w$) values in two, two, and four instances, while Cyclic-Fast achieved the highest \(avg_w\) values in only two instances.

\begin{table*}[!h]
\centering
\caption{Experimental results on Network.}
\resizebox{1.0\linewidth}{!}{
\begin{tabular}{l|llllllll} 
 \hline
\multirow{2}{*}{\textbf{Graph}} &
\multicolumn{2}{c}{\textbf{DynLS}} &
\multicolumn{2}{c}{\textbf{Cyclic-Fast}} &
\multicolumn{2}{c}{\textbf{HILS}} &
\multicolumn{1}{c}{\textbf{Solve}} &
\multicolumn{1}{c}{\textbf{HtWIS}} \\

\cmidrule(r){2-9} & \multicolumn{1}{c}{\textbf{$max_w$}} & \multicolumn{1}{c}{\textbf{$avg_w$}} & \multicolumn{1}{c}{\textbf{$max_w$}} & \multicolumn{1}{c}{\textbf{$avg_w$}} & \multicolumn{1}{c}{\textbf{$max_w$}} & \multicolumn{1}{c}{\textbf{$avg_w$}} & \multicolumn{1}{c}{\textbf{$value_w$}} & \multicolumn{1}{c}{\textbf{$value_w$}} \\
\hline
rec-amazon & \textbf{5108939} & \textbf{5108939} & \textbf{5108939} & \textbf{5108939} & 5101076 & 5088857 & \textbf{5108939} & 5108886 \\[-1pt]
retweet-crawl & \textbf{113078098} & \textbf{113078098} & \textbf{113078098} & \textbf{113078098} & N/A & N/A & \textbf{113078098} & \textbf{113078098} \\[-1pt]
road-central & \textbf{830232458} & \textbf{830232458} & 830232183 & 830228951.80 & N/A & N/A & 830055255 & 830012442 \\[-1pt]
roadNet-CA & \textbf{101899948} & \textbf{101899371.90} & 101731370 & 101727657.40 & N/A & N/A & 100887364 & 101577452 \\[-1pt]
roadNet-PA & \textbf{57105569} & \textbf{57105436.20} & 57052752 & 57030798.40 & 49795448 & 49727618.20 & 56632742 & 56929109 \\[-1pt]
road-usa & \textbf{1381590406} & \textbf{1381589870} & 1381530298 & 1381521103.60 & N/A & N/A & 1380608075 & 1380851928 \\[-1pt]
rt-retweet-crawl & \textbf{103909643} & \textbf{103909643} & \textbf{103909643} & \textbf{103909643} & 103168896 & 102937712.80 & \textbf{103909643} & \textbf{103909643} \\[-1pt]
sc-ldoor & \textbf{10302330} & \textbf{10302330} & \textbf{10302330} & \textbf{10302330} & 9489602 & 9343420.80 & \textbf{10302330} & 10296910 \\[-1pt]
sc-msdoor & \textbf{3912203} & \textbf{3912203} & \textbf{3912203} & \textbf{3912203} & 3791244 & 3749901.60 & 3902981 & 3906819 \\[-1pt]
sc-nasasrb & \textbf{441645} & \textbf{441634.50} & 436859 & 427569.80 & 435459 & 434886 & 410028 & 434430 \\[-1pt]
sc-pkustk11 & \textbf{523520} & \textbf{523520} & \textbf{523520} & \textbf{523520} & 519122 & 518636.60 & 517690 & 521540 \\[-1pt]
sc-pkustk13 & \textbf{730705} & \textbf{730705} & \textbf{730705} & \textbf{730705} & 723091 & 721704 & 690196 & 721338 \\[-1pt]
sc-pwtk & \textbf{1199520} & \textbf{1199228.50} & 1188112 & 1185867.60 & 1147405 & 1135962.20 & 1147964 & 1173203 \\[-1pt]
sc-shipsec1 & \textbf{2878129} & \textbf{2877664.50} & 2641278 & 2626805 & 2777110 & 2733567 & 2552725 & 2775670 \\[-1pt]
sc-shipsec5 & \textbf{3778449} & \textbf{3777205.70} & 3474214 & 3450373.40 & 3599490 & 3524863 & 3255839 & 3673203 \\[-1pt]
soc-BlogCatalog & \textbf{6997002} & \textbf{6997002} & \textbf{6997002} & \textbf{6997002} & 6993717 & 6990306.40 & \textbf{6997002} & \textbf{6997002} \\[-1pt]
soc-brightkite & \textbf{3821557} & \textbf{3821557} & \textbf{3821557} & \textbf{3821557} & 3820470 & 3820014.60 & \textbf{3821557} & \textbf{3821557} \\[-1pt]
soc-buzznet & \textbf{7323353} & \textbf{7323353} & 7323273 & 7323257.20 & 7311286 & 7309515.40 & \textbf{7323353} & 7322761 \\[-1pt]
soc-delicious & \textbf{45946372} & \textbf{45946372} & \textbf{45946372} & \textbf{45946372} & 45373671 & 45210481.40 & \textbf{45946372} & 45941169 \\[-1pt]
soc-digg & \textbf{67761717} & \textbf{67761717} & \textbf{67761717} & \textbf{67761717} & 66938402 & 66770321.60 & \textbf{67761717} & 67761681 \\[-1pt]
soc-epinions & \textbf{1814132} & \textbf{1814132} & \textbf{1814132} & \textbf{1814132} & 1813973 & 1813798 & \textbf{1814132} & \textbf{1814132} \\[-1pt]
soc-flickr & \textbf{37783767} & \textbf{37783761.40} & 37783720 & 37783664.40 & 36792167 & 36773484.40 & \textbf{37783767} & 37783282 \\[-1pt]
soc-flixster & \textbf{244070588} & \textbf{244070588} & \textbf{244070588} & \textbf{244070588} & 242693692 & 242625212.40 & \textbf{244070588} & \textbf{244070588} \\[-1pt]
soc-FourSquare & \textbf{55452016} & \textbf{55452016} & \textbf{55452016} & \textbf{55452016} & 55339222 & 55273988.40 & \textbf{55452016} & \textbf{55452016} \\[-1pt]
soc-gowalla & \textbf{12202367} & \textbf{12202367} & \textbf{12202367} & \textbf{12202367} & 12046758 & 12041513.80 & \textbf{12202367} & 12202004 \\[-1pt]
soc-lastfm & \textbf{111998876} & \textbf{111998876} & \textbf{111998876} & \textbf{111998876} & 111364538 & 111340631.40 & \textbf{111998876} & \textbf{111998876} \\[-1pt]
soc-LiveMocha & \textbf{6430396} & \textbf{6430396} & \textbf{6430396} & \textbf{6430396} & 6405796 & 6405269 & \textbf{6430396} & 6430394 \\[-1pt]
soc-slashdot & \textbf{5077946} & \textbf{5077946} & \textbf{5077946} & \textbf{5077946} & 5076917 & 5076745.20 & \textbf{5077946} & \textbf{5077946} \\[-1pt]
soc-youtube & \textbf{37007102} & \textbf{37007102} & \textbf{37007102} & \textbf{37007102} & 35946856 & 35934347 & \textbf{37007102} & 37007097 \\[-1pt]
soc-youtube-snap & \textbf{90174634} & \textbf{90174634} & \textbf{90174634} & \textbf{90174634} & 86608040 & 86560424.40 & \textbf{90174634} & 90174629 \\[-1pt]
tech-as-skitter & \textbf{124032626} & \textbf{124032626} & 124028142 & 124028046.20 & N/A & N/A & 124028079 & 124010835 \\[-1pt]
tech-p2p-gnutella & \textbf{4754823} & \textbf{4754823} & \textbf{4754823} & \textbf{4754823} & 4754716 & 4754622 & \textbf{4754823} & \textbf{4754823} \\[-1pt]
tech-RL-caida & \textbf{12299662} & \textbf{12299662} & \textbf{12299662} & \textbf{12299662} & 12158691 & 12093637 & \textbf{12299662} & 12294041 \\[-1pt]
web-arabic2005 & \textbf{5073754} & \textbf{5073754} & \textbf{5073754} & \textbf{5073754} & 5056166 & 5043940.20 & \textbf{5073754} & \textbf{5073754} \\[-1pt]
web-BerkStan & \textbf{723837} & \textbf{723837} & \textbf{723837} & \textbf{723837} & \textbf{723837} & \textbf{723837} & \textbf{723837} & 723793 \\[-1pt]
web-edu & \textbf{162434} & \textbf{162434} & \textbf{162434} & \textbf{162434} & \textbf{162434} & 162432 & \textbf{162434} & \textbf{162434} \\[-1pt]
web-indochina-2004 & \textbf{415965} & \textbf{415965} & \textbf{415965} & \textbf{415965} & 415220 & 415031 & \textbf{415965} & \textbf{415965} \\[-1pt]
web-it2004 & \textbf{10298594} & \textbf{10298594} & \textbf{10298594} & \textbf{10298594} & 10040966 & 10004648.20 & \textbf{10298594} & 10296599 \\[-1pt]
web-sk2005 & \textbf{6746271} & \textbf{6746271} & 6746151 & 6746014.40 & 6736740 & 6725276.60 & 6746168 & 6743987 \\[-1pt]
web-spam & \textbf{267817} & \textbf{267817} & \textbf{267817} & \textbf{267817} & 267806 & 267805 & \textbf{267817} & 267809 \\[-1pt]
web-uk2005 & \textbf{265384} & \textbf{265384} & \textbf{265384} & \textbf{265384} & N/A & N/A & \textbf{265384} & \textbf{265384} \\[-1pt]
web-webbase2001 & \textbf{1355649} & \textbf{1355649} & \textbf{1355649} & \textbf{1355649} & \textbf{1355649} & 1355643.80 & \textbf{1355649} & 1355636 \\[-1pt]
web-wikipedia2009 & \textbf{130289852} & \textbf{130289834.50} & 130276367 & 130275393.20 & N/A & N/A & 130273301 & 130285876 \\

\hline
total & \textbf{43} & \textbf{43} & 30 & 30 & 3 & 1 & 29 & 14 \\
\hline
\end{tabular}}

\label{Experimental-results-on-network}
\end{table*}

% \vspace{-0.2cm}

\begin{table*}[!h]
\centering
\caption{Experimental results on Fe.}
\resizebox{1.0\linewidth}{!}{
\begin{tabular}{l|lllllllll} 
\hline
\multirow{2}{*}{\textbf{Graph}} &
\multicolumn{2}{c}{\textbf{DynLS}} &
\multicolumn{2}{c}{\textbf{Cyclic-Fast}} &
\multicolumn{2}{c}{\textbf{HILS}} &
\multicolumn{1}{c}{\textbf{m$^2$wis + s}} &
\multicolumn{1}{c}{\textbf{Solve}} &
\multicolumn{1}{c}{\textbf{HtWIS}}
 \\   

\cmidrule(r){2-10} & \multicolumn{1}{c}{\textbf{$max_w$}} & \multicolumn{1}{c}{\textbf{$avg_w$}} & \multicolumn{1}{c}{\textbf{$max_w$}} & \multicolumn{1}{c}{\textbf{$avg_w$}} & \multicolumn{1}{c}{\textbf{$max_w$}} & \multicolumn{1}{c}{\textbf{$avg_w$}} & \multicolumn{1}{c}{\textbf{$value_w$}} & \multicolumn{1}{c}{\textbf{$value_w$}} & \multicolumn{1}{c}{\textbf{$value_w$}} \\
\hline

fe\_4elt2\hspace{1.55cm} & \textbf{428029} & \textbf{428029} & \textbf{428029} & \textbf{428029} & 427042 & 426766.60 & N/A & N/A & 423824 \\[-1pt]
fe\_body & \textbf{1680182} & \textbf{1680182} & \textbf{1680182} & 1680110.80 & N/A & N/A & 1680166 & 1670783 & 1675435 \\[-1pt]
fe\_ocean & 7175687 & \textbf{7162419.10} & 6412128 & 6359148.60 & 6921036 & 6908666.80 & \textbf{7248581} & \textbf{7248581} & 6803672 \\[-1pt]
fe\_pwt & \textbf{1178683} & \textbf{1178580.20} & 1161745 & 1156460.20 & N/A & N/A & 1178434 & 1116090 & 1158822 \\[-1pt]
fe\_rotor & \textbf{2664325} & \textbf{2663756.40} & 2513611 & 2506356.60 & N/A & N/A & 2642600 & N/A & 2591456 \\[-1pt]
fe\_sphere & \textbf{617816} & \textbf{617816} & \textbf{617816} & \textbf{617816} & N/A & N/A & \textbf{617816} & 595507 & 608401 \\[-1pt]
fe\_tooth & \textbf{3033298} & \textbf{3033298} & \textbf{3033298} & \textbf{3033298} & 3029582 & 3028853 & N/A & 3028675 & 3031490 \\

\hline
total & \textbf{6} & \textbf{7} & 4 & 3 & 0 & 0 & 2 & 1 & 0 \\
\hline
\end{tabular}}

\label{Experimental-results-on-fe}
\end{table*}

Concerning the Sparse dataset, DynLS still demonstrated outstanding performance. It achieved the highest \(max_w\) values in 44 (out of 45) instances and the highest \(avg_w\) values in 42 instances, surpassing all other algorithms, as illustrated in Table S5 of the supplementary file Section S5. The Cyclic-Fast algorithm secured the second position, attaining the highest \(max_w\) and \(avg_w\) values in 35 instances. Similarly, the Solve algorithm performed comparable to Cyclic-Fast, achieving the highest $value_w$ values in 27 instances. In contrast, the performance of other algorithms was less impressive; the HILS algorithm achieved only four highest \(max_w\) values and four highest \(avg_w\) values, whereas the HtWIS algorithm recorded 10 highest $value_w$ values.

\subsubsection{Comparison of solution quality on Family B}
In this family of datasets, the DynLS algorithm was compared to all five state-of-the-art algorithms: HILS, HtWIS, Solve, Cyclic-Fast, and m$^2$wis + s. Across the various datasets, DynLS consistently exhibited notable outperformance relative to the other algorithms. Specifically, DynLS achieved the highest values for both \(max_w\) and \(avg_w\) across all instances in the Snap and Ssmc benchmarks. In comparison, other algorithms often fell short of reaching the highest \(max_w\) (or $value_w$) or \(avg_w\) values in certain instances, particularly within large-scale datasets, as demonstrated in Tables \ref{Experimental-results-on-snap} and \ref{Experimental-results-on-ssmc}. Next, we will outline the performance results of these algorithms on additional datasets.

\begin{table*}[!h]
\centering
\caption{Experimental results on Snap.}
\resizebox{1.0\linewidth}{!}{
\begin{tabular}{l|lllllllll} 
\hline
\multirow{2}{*}{\textbf{Graph}} &
\multicolumn{2}{c}{\textbf{DynLS}} &
\multicolumn{2}{c}{\textbf{Cyclic-Fast}} &
\multicolumn{2}{c}{\textbf{HILS}} &
\multicolumn{1}{c}{\textbf{m$^2$wis + s}} &
\multicolumn{1}{c}{\textbf{Solve}} &
\multicolumn{1}{c}{\textbf{HtWIS}}
 \\   

\cmidrule(r){2-10} & \multicolumn{1}{c}{\textbf{$max_w$}} & \multicolumn{1}{c}{\textbf{$avg_w$}} & \multicolumn{1}{c}{\textbf{$max_w$}} & \multicolumn{1}{c}{\textbf{$avg_w$}} & \multicolumn{1}{c}{\textbf{$max_w$}} & \multicolumn{1}{c}{\textbf{$avg_w$}} & \multicolumn{1}{c}{\textbf{$value_w$}} & \multicolumn{1}{c}{\textbf{$value_w$}} & \multicolumn{1}{c}{\textbf{$value_w$}} \\
\hline
as-skitter\hspace{1.9cm} & \textbf{124157729} & \textbf{124157729} & 124154023 & 124154008 & N/A & N/A & 124157712 & 124153425 & 124141373 \\[-1pt]
ca-AstroPh & \textbf{797510} & \textbf{797510} & \textbf{797510} & \textbf{797510} & 797500 & 797411.60 & \textbf{797510} & \textbf{797510} & \textbf{797510} \\[-1pt]
ca-CondMat & \textbf{1147950} & \textbf{1147950} & \textbf{1147950} & \textbf{1147950} & 1147929 & 1147848.60 & \textbf{1147950} & \textbf{1147950} & \textbf{1147950} \\[-1pt]
ca-GrQc & \textbf{286489} & \textbf{286489} & \textbf{286489} & \textbf{286489} & \textbf{286489} & \textbf{286489} & \textbf{286489} & \textbf{286489} & \textbf{286489} \\[-1pt]
ca-HepPh & \textbf{581039} & \textbf{581039} & \textbf{581039} & \textbf{581039} & \textbf{581039} & 581030.20 & \textbf{581039} & \textbf{581039} & \textbf{581039} \\[-1pt]
ca-HepTh & \textbf{562004} & \textbf{562004} & \textbf{562004} & \textbf{562004} & \textbf{562004} & 561994.80 & \textbf{562004} & \textbf{562004} & \textbf{562004} \\[-1pt]
email-Enron & \textbf{2464935} & \textbf{2464935} & \textbf{2464935} & \textbf{2464935} & 2464817 & 2464677.40 & \textbf{2464935} & \textbf{2464935} & \textbf{2464935} \\[-1pt]
email-EuAll & \textbf{25286322} & \textbf{25286322} & \textbf{25286322} & \textbf{25286322} & 25286088 & 25285402.20 & \textbf{25286322} & \textbf{25286322} & \textbf{25286322} \\[-1pt]
p2p-G.04 & \textbf{679111} & \textbf{679111} & \textbf{679111} & \textbf{679111} & \textbf{679111} & 679097.80 & \textbf{679111} & \textbf{679111} & 679085 \\[-1pt]
p2p-G.05 & \textbf{554943} & \textbf{554943} & \textbf{554943} & \textbf{554943} & 554856 & 554846.40 & \textbf{554943} & \textbf{554943} & \textbf{554943} \\[-1pt]
p2p-G.06 & \textbf{548612} & \textbf{548612} & \textbf{548612} & \textbf{548612} & \textbf{548612} & 548609.60 & \textbf{548612} & \textbf{548612} & \textbf{548612} \\[-1pt]
p2p-G.08 & \textbf{434577} & \textbf{434577} & \textbf{434577} & \textbf{434577} & \textbf{434577} & \textbf{434577} & \textbf{434577} & \textbf{434577} & \textbf{434577} \\[-1pt]
p2p-G.09 & \textbf{568439} & \textbf{568439} & \textbf{568439} & \textbf{568439} & \textbf{568439} & \textbf{568439} & \textbf{568439} & \textbf{568439} & \textbf{568439} \\[-1pt]
p2p-G.24 & \textbf{1984567} & \textbf{1984567} & \textbf{1984567} & \textbf{1984567} & 1984556 & 1984532.80 & \textbf{1984567} & \textbf{1984567} & \textbf{1984567} \\[-1pt]
p2p-G.25 & \textbf{1701967} & \textbf{1701967} & \textbf{1701967} & \textbf{1701967} & \textbf{1701967} & \textbf{1701967} & \textbf{1701967} & \textbf{1701967} & \textbf{1701967} \\[-1pt]
p2p-G.30 & \textbf{2787907} & \textbf{2787907} & \textbf{2787907} & \textbf{2787907} & 2787897 & 2787850 & \textbf{2787907} & \textbf{2787907} & 2787902 \\[-1pt]
p2p-G.31 & \textbf{4776986} & \textbf{4776986} & \textbf{4776986} & \textbf{4776986} & 4776923 & 4776853 & \textbf{4776986} & \textbf{4776986} & 4776925 \\[-1pt]
roadNet-CA & \textbf{111360828} & \textbf{111360828} & \textbf{111360828} & \textbf{111360828} & N/A & N/A & \textbf{111360828} & \textbf{111360828} & 111325524 \\[-1pt]
roadNet-PA & \textbf{61731589} & \textbf{61731589} & \textbf{61731589} & \textbf{61731589} & 54128628 & 54027129.20 & \textbf{61731589} & \textbf{61731589} & 61710606 \\[-1pt]
roadNet-TX & \textbf{78599946} & \textbf{78599946} & \textbf{78599946} & \textbf{78599946} & N/A & N/A & \textbf{78599946} & \textbf{78599946} & 78575460 \\[-1pt]
soc-Epinions1 & \textbf{5690970} & \textbf{5690970} & \textbf{5690970} & \textbf{5690970} & 5690246 & 5689963 & \textbf{5690970} & \textbf{5690970} & \textbf{5690970} \\[-1pt]
soc-LiveJournal1 & \textbf{284036239} & \textbf{284036239} & 284036232 & 284036174.80 & N/A & N/A & 284036236 & 284021396 & 284019992 \\[-1pt]
s.p.r. & \textbf{83979964} & \textbf{83977208.90} & N/A & N/A & N/A & N/A & 83720972 & 79366658 & 83920370 \\[-1pt]
soc-S.0811 & \textbf{5660899} & \textbf{5660899} & \textbf{5660899} & \textbf{5660899} & 5659398 & 5658825.60 & \textbf{5660899} & \textbf{5660899} & \textbf{5660899} \\[-1pt]
soc-S.0902 & \textbf{5971849} & \textbf{5971849} & \textbf{5971849} & \textbf{5971849} & 5969604 & 5968399 & \textbf{5971849} & \textbf{5971849} & 5971821 \\[-1pt]
web-BerkStan & \textbf{43907482} & \textbf{43907482} & \textbf{43907482} & \textbf{43907482} & 40947745 & 40722845 & \textbf{43907482} & 43899932 & 43889843 \\[-1pt]
web-Google & \textbf{56326504} & \textbf{56326504} & \textbf{56326504} & \textbf{56326504} & 52329458 & 51987485.20 & \textbf{56326504} & \textbf{56326504} & 56323382 \\[-1pt]
web-NotreDame & \textbf{26016941} & \textbf{26016941} & \textbf{26016941} & \textbf{26016941} & 25818162 & 25762580.60 & \textbf{26016941} & \textbf{26016941} & 26013830 \\[-1pt]
web-Stanford & \textbf{17792930} & \textbf{17792930} & \textbf{17792930} & \textbf{17792930} & 17073222 & 16974535 & \textbf{17792930} & \textbf{17792930} & 17789430 \\[-1pt]
wiki-Talk & \textbf{235837346} & \textbf{235837346} & \textbf{235837346} & \textbf{235837346} & 235401815 & 235377297.60 & \textbf{235837346} & \textbf{235837346} & \textbf{235837346} \\[-1pt]
wiki-Vote & \textbf{500079} & \textbf{500079} & \textbf{500079} & \textbf{500079} & 500057 & 500047.20 & \textbf{500079} & \textbf{500079} & 499740 \\
\hline
total & \textbf{31} & \textbf{31} & 28 & 28 & 8 & 4 & 28 & 27 & 16 \\
\hline
\end{tabular}}

\label{Experimental-results-on-snap}
\end{table*}

% \vspace{-0.2cm}

\begin{table*}[!h]
\centering
\caption{Experimental results on Ssmc.}
\resizebox{1.0\linewidth}{!}{
\begin{tabular}{l|lllllllll} 
\hline
\multirow{2}{*}{\textbf{Graph}} &
\multicolumn{2}{c}{\textbf{DynLS}} &
\multicolumn{2}{c}{\textbf{Cyclic-Fast}} &
\multicolumn{2}{c}{\textbf{HILS}} &
\multicolumn{1}{c}{\textbf{m$^2$wis + s}} &
\multicolumn{1}{c}{\textbf{Solve}} &
\multicolumn{1}{c}{\textbf{HtWIS}}
 \\   

\cmidrule(r){2-10} & \multicolumn{1}{c}{\textbf{$max_w$}} & \multicolumn{1}{c}{\textbf{$avg_w$}} & \multicolumn{1}{c}{\textbf{$max_w$}} & \multicolumn{1}{c}{\textbf{$avg_w$}} & \multicolumn{1}{c}{\textbf{$max_w$}} & \multicolumn{1}{c}{\textbf{$avg_w$}} & \multicolumn{1}{c}{\textbf{$value_w$}} & \multicolumn{1}{c}{\textbf{$value_w$}} & \multicolumn{1}{c}{\textbf{$value_w$}} \\
\hline
ca2010\hspace{1.9cm} & \textbf{16869550} & \textbf{16869428.80} & 16862232 & 16855924.80 & 14811354 & 14688771.20 & \textbf{16869550} & 16712752 & 16792827 \\[-1pt]
fl2010 & \textbf{8743506} & \textbf{8743506} & \textbf{8743506} & 8743096 & 8146034 & 8133111.80 & \textbf{8743506} & 8680997 & 8719272 \\[-1pt]
ga2010 & \textbf{4644417} & \textbf{4644417} & \textbf{4644417} & \textbf{4644417} & 4448397 & 4402020 & \textbf{4644417} & \textbf{4644417} & 4639891 \\[-1pt]
il2010 & \textbf{5998539} & \textbf{5998539} & \textbf{5998539} & \textbf{5998539} & 5391544 & 5285554.40 & \textbf{5998539} & 5914067 & 5963974 \\[-1pt]
nh2010 & \textbf{588996} & \textbf{588996} & \textbf{588996} & \textbf{588996} & 587146 & 586840.80 & \textbf{588996} & \textbf{588996} & 587059 \\[-1pt]
ri2010 & \textbf{459275} & \textbf{459275} & \textbf{459275} & \textbf{459275} & 456950 & 456484 & \textbf{459275} & \textbf{459275} & 457108 \\
\hline
total & \textbf{6} & \textbf{6} & 5 & 4 & 0 & 0 & \textbf{6} & 3 & 0 \\
\hline
\end{tabular}}

\label{Experimental-results-on-ssmc}
\end{table*}

The results for the Fe benchmark were presented in Table \ref{Experimental-results-on-fe}. In this dataset, the DynLS algorithm attained the highest \(max_w\) and \(avg_w\) values for all instances, except one instance ``fe\_ocean.'' The Cyclic-fast algorithm demonstrated suboptimal performance, achieving the highest \(max_w\) values in four instances and the highest \(avg_w\) values in three instances. In contrast, other algorithms did not yield satisfactory results. Notably, the HILS algorithm was unable to produce feasible solutions for several instances within the 1200-second time limit, including ``fe\_body,'' ``fe\_pwt,'' ``fe\_rotor,'' and ``fe\_sphere,'' and it did not obtain any highest \(max_w\) or \(avg_w\) value.

The findings for the Osm benchmark were outlined in Table S7 of the supplementary file Section S5. In this dataset, the m$^2$wis + s algorithm consistently attained the highest $value_w$ values across all 34 instances, exceeding the performance of DynLS on four specific instances: ``d.o.c.3,'' ``hawaii3,'' ``kentucky3,'' and ``rhode-i.3.'' Nonetheless, the discrepancies in the solutions generated by DynLS and m$^2$wis + s were marginal, indicating that DynLS demonstrated performance that is essentially comparable to that of m$^2$wis + s within this dataset. Following these two algorithms, Cyclic-Fast achieved the highest $max_w$ values in 27 instances. In contrast, the other three algorithms exhibited subpar performance, with HILS and Solve recording the highest $max_w$ or $value_w$ values in 9 and 22 instances, respectively. Notably, HtWIS did not identify any of the highest $value_w$ values. Regarding average performance, DynLS outperformed HILS and Cyclic-Fast, attaining the highest  $avg_w$ values in 32 instances. Conversely, Cyclic-Fast and HILS achieved the highest $avg_w$ values in 21 and 5 instances, respectively. This result underscores the stability and reliability of the DynLS algorithm in its application to this dataset.

Finally, the results for the Mesh benchmark are presented in Table S6 of the supplementary file Section S5. In this dataset, the unique structure of the instances facilitated the effectiveness of the reduction rules in DynLS, Cyclic-Fast, and Solve. As a result, these three algorithms achieved the highest \(max_w\) (or $value_w$) values across all 15 instances. These solutions are optimal, given that Solve is an exact algorithm. The m$^2$wis + s algorithm demonstrated comparable performance, securing the best solutions in 14 instances. In contrast, the HILS and HtWIS algorithms underperformed. Regarding average performance, both the DynLS and Cyclic-Fast algorithms recorded the highest \(avg_w\) values for all instances. In contrast, the HILS algorithm struggled, failing to obtain the highest \(avg_w\) values for any instance.

% \begin{figure*}[h]
    
%     \centerline{\includegraphics[width=0.8\linewidth]{figures/1.jpg}}
%     \caption{The heatmaps in Figures (a) and (b) represent the average ranking of each algorithm in family A and family B, respectively. A smaller value indicates better algorithm performance.}
%     \label{fig1:reli}
% \end{figure*}

In summary, the DynLS algorithm consistently outperforms all baseline algorithms across various datasets, offering superior best and average solutions in nearly all cases. The sole exception is that DynLS shows slightly lower performance than the m$^2$wis + s algorithm on the Osm benchmark regarding the best solution. Specifically, DynLS achieved the highest $max_w$ values in 262 instances in Family A, surpassing HILS by 28 instances, HtWIS by 30, Solve by 76, and Cyclic-Fast by 98. When considering average solutions, DynLS excelled in these 261 instances, leading over HILS by 17 instances and Cyclic-Fast by 92. In Family B, DynLS attained the highest $max_w$ values in 88 instances, exceeding HILS by 17 instances, HtWIS by 16, Solve by 68, Cyclic-Fast by 79, and m$^2$wis + s by 84. For average solutions, DynLS achieved the highest $avg_w$ values in 91 instances, outperforming HILS by nine instances and Cyclic-Fast by 71. These findings illustrate that the DynLS algorithm can produce high-quality solutions across diverse instances, consistently surpassing the performance of other established algorithms.

\vspace{-0.3cm}
\subsection{Comparison of convergence time}
The convergence time, represented as ``$time$,'' is the minimum duration required for an algorithm to compute its final solution for a specific instance. To demonstrate the effectiveness of our DynLS algorithm, we compared the convergence times of all algorithms across various benchmark instances. For fairness, we focused only on instances where at least two algorithms achieved the maximum values of \( max_w \) or \( value_w \). We analyzed 195 instances, comprising 111 from family A and 84 from family B.

Table S8 of the supplementary file Section S6 details the convergence times of DynLS, Cyclic-Fast, HILS, Solve, and HtWIS across instances from family A. Among others, the Cyclic-Fast algorithm stands out by achieving the shortest convergence time in 91 instances, all of which are less than 10 seconds. However, it is worth noting that Cyclic-Fast did not attain the highest $max_w$ value in 13 instances. In contrast, while DynLS had fewer instances with the shortest convergence times, the convergence times for DynLS are very close to those of Cyclic-Fast within this family, with 87 instances recording convergence times under 10 seconds. Importantly, DynLS did not fail to achieve the highest $max_w$ value in any instance. The performance of the other algorithms was comparatively weaker. The Solve and HILS algorithms secured the shortest convergence times in two and five instances, respectively. However, HILS generally exhibited longer convergence times across most instances, often extending into several hundred seconds. Meanwhile, the HtWIS algorithm fell short overall, failing to achieve the shortest convergence time in any instance.

Table S9 of the supplementary file Section S6 provides the convergence times for various algorithms across instances from family B. Again, the Cyclic-Fast algorithm demonstrated superior performance, achieving the shortest convergence time in 72 instances and in 75 instances with convergence times of less than 10 seconds. However, it did not attain the highest $max_w$ value in five instances. The DynLS algorithm ranked second, achieving the shortest convergence time in 12 instances. Even in cases where it did not secure the shortest time, its performance was very close to that of the Cyclic-Fast algorithm, with 73 instances recording convergence times under 10 seconds. DynLS failed to achieve the highest $max_w$ value in one instance. In contrast, the HILS, HtWIS, and Solve algorithms did not achieve the best convergence time in any instances and generally performed poorly. Furthermore, while the m$^2$wis + s algorithm achieved the best convergence time in one instance, it exceeded 1000 seconds in many other cases.

Overall, when evaluating the instances with the minimum convergence times, the Cyclic-Fast algorithm outperforms the DynLS algorithm, while DynLS surpasses other algorithms. Although DynLS does not identify as many instances with the minimum convergence time as Cyclic-Fast, the difference in their convergence times is minimal. Moreover, the Cyclic-Fast algorithm frequently struggles to achieve the highest \(max_w\) values across various instances. In 165 instances, the Cyclic-Fast algorithm significantly trails behind DynLS regarding the \(max_w\) value. This suggests that the DynLS algorithm has a substantial advantage in achieving the best solutions and excels in convergence times.

\subsection{Ablation Experiments}\label{sec5-4}
We performed two ablation experiments to evaluate each component's performance and contributions within the DynLS algorithm. Specifically, we removed the following two components from the DynLS algorithm: one is the scores-based adaptive vertex perturbation (SAVP) strategy, and the other is the region location mechanism and the reward-based complex vertex exchange module (RLM \& ComLS, as detailed in lines 11-29 and line 28 of Algorithm \ref{alg-main}). This approach yielded two algorithm variants: DynLS-SAVP and DynLS-RLM \& ComLS, which we compared to the original DynLS algorithm. All experiments were executed over 1000 seconds, with each instance run independently five times using seeds 1 through 5. For each instance, we reported both the best solution (\( max_w \)) and the average solution (\( avg_w \)) from the five trials. In the results table, we included only those instances where \( max_w \) values differed between the compared algorithms. 

Table S10 of the supplementary file Section S7 compares the DynLS algorithm with its DynLS-SAVP variant. In analyzing 106 retained instances, DynLS achieved the highest $max_w$ value across all instances, whereas DynLS-SAVP recorded a $max_w$ value of zero in every case. For the $avg_w$ metric, DynLS delivered high-quality solutions in 99 instances, while DynLS-SAVP provided solutions in only seven instances. These results highlight the SAVP component's considerable contribution to the DynLS algorithm's overall efficacy.
Table S11 of the supplementary file Section S7 illustrates the comparative performance of DynLS and DynLS-RLM\&ComLS. Among the 103 instances retained for review, DynLS achieved the highest $max_w$ value in 94 instances, while DynLS-RLM\&ComLS reached this mark in only nine instances. Regarding $avg_w$, DynLS produced high-quality solutions in 77 instances, compared to DynLS-RLM\&ComLS, which provided 26 high-quality solutions. These findings indicate that the region location mechanism and reward-based complex vertex exchange modules integrated within DynLS demonstrate superior performance in many cases, yielding a greater quantity and higher quality of solutions. 

% \begin{figure}[h]
%     \hspace{0.4cm}
%     \centerline{\includegraphics[width=0.8\linewidth]{figures/A.png}}
%     \caption{The heatmaps represent the average ranking of each algorithm in family A, respectively. A smaller value indicates better algorithm performance.}
%     \label{fig1:reliA}
% \end{figure}
% \begin{figure}[h]
%     \hspace{0.4cm}
%     \centerline{\includegraphics[width=0.8\linewidth]{figures/B.png}}
%     \caption{The heatmaps represent the average ranking of each algorithm in family B, respectively. A smaller value indicates better algorithm performance.}
%     \label{fig1:reliB}
% \end{figure}

In conclusion, the findings from the various experiments demonstrate that the components of DynLS are essential and irreplaceable. These components collectively establish the DynLS algorithm's foundation, significantly enhancing its effectiveness in solving the MWIS problem.

\subsection{Statistical test}
We employ the Friedman test \cite{raj2023novel} to further validate the effectiveness of the DynLS algorithm. The Friedman test is a non-parametric statistical method used to analyze whether there are significant differences in the performance of multiple algorithms across multiple test instances. In this study, we rank the algorithms for each instance based on the quality of their solutions (with higher quality receiving better rankings). The final ranking for a dataset is obtained by averaging the rankings across all instances within that dataset. An algorithm's overall ranking (represented by ``Overall rank'') is then calculated as the mean of its final rankings across all datasets. 

\vspace{-0.2cm}

\begin{table}[H]
\centering
\caption{Friedman test results for ranking compared algorithms based on average ranking of solution quality}
\resizebox{1\linewidth}{!}{
\begin{tabular}{l|cccccc} 
\hline
\multirow{2}{*}{Benchmark} &
\multicolumn{6}{c}{Algorithm}  \\
\cmidrule{2-7} & 
\multicolumn{1}{c}{\textbf{DynLS}} &
\multicolumn{1}{c}{\textbf{Cyclic-Fast}} &
\multicolumn{1}{c}{\textbf{HILS}} &
\multicolumn{1}{c}{\textbf{m$^2$wis + s}} &
\multicolumn{1}{c}{\textbf{Solve}} &
\multicolumn{1}{c}{\textbf{HtWIS}} \\ 
\hline
Dimacs & \textbf{1.03} & 2.27 & 3.29 & - & 3.24 &  2.97  \\ [-1pt]
Facebook & \textbf{1.02} & 2.94 & 1.96 & - & 4.93 &  3.92 \\ [-1pt]
Network & \textbf{1.00} & 1.49 & 4.49 & - & 2.00 &  2.81 \\ [-1pt]
Sparse & \textbf{1.02} & 1.42 & 4.18 & - & 2.22 &  3.38 \\
\hline
Overall rank & \textbf{1.0175} & 2.03 & 3.48 & - & 3.0975 & 3.27 \\
\hline
Fe & \textbf{1.29} & 2.43 & 4.86 & 2.86 & 4.23 & 3.86 \\ [-1pt]
Mesh & \textbf{1.00} & \textbf{1.00} & 5.73 & 1.33 & \textbf{1.00} & 5.13 \\ [-1pt]

Osm & 1.21 & 1.68 & 3.38 & \textbf{1.00} & 2.29 & 5.85 \\ [-1pt]
Snap & \textbf{1.00} & 1.26 & 4.65 & 1.13 & 1.39 & 2.90 \\ [-1pt]
Ssmc & \textbf{1.00} & 1.33 & 5.83 & \textbf{1.00} & 3.00 & 4.67 \\
\hline
Overall rank & \textbf{1.1} & 1.54 & 4.89 & 1.464 & 2.382 & 4.482 \\
\hline
\end{tabular}}

\label{overall}
\end{table}

The analysis results are presented in Table \ref{overall}, while the corresponding ranking results are illustrated in Fig. S1 (a) and Fig. S1 (b) in Section S8 of the supplementary file. 
{The area and color of the squares in the figure both represent the algorithm's average ranking performance on the dataset. A smaller area combined with a bluer hue indicates a lower average ranking, while a larger area with a redder color signifies a higher average ranking. It is evident from the figure that the squares corresponding to the DynLS algorithm generally exhibit smaller areas and a more pronounced blue hue across all datasets, with the exception of Osm. This suggests that the DynLS algorithm tends to achieve lower rankings, demonstrating the relatively high quality of the solutions it provides. 
}
Specifically, DynLS achieved the lowest rankings across all datasets in family A, with rankings of 1.03, 1.02, 1.00, and 1.02 for the Dimacs, Facebook, Network, and Sparse datasets, respectively. Its overall rank of 1.0175 was the smallest among all algorithms, demonstrating superior performance. The overall ranks of the other algorithms were as follows: Cyclic-Fast (2.03), HtWIS (3.27), Solve (3.0975), and HILS (3.48). In family B, DynLS also attained the lowest rankings on the Fe, Mesh, Snap, and Ssmc datasets, with rankings of 1.29, 1.00, 1.00, and 1.00, respectively. On the Osm dataset, the m$^2$wis + s algorithm achieved the lowest ranking (1.00), and both m$^2$wis+s and DynLS shared the lowest ranking on the Ssmc dataset. Additionally, Cyclic-Fast and Solve tied with DynLS for the lowest ranking on the Mesh dataset. Considering the overall rank, DynLS again outperformed the other algorithms with the smallest overall rank of 1.1. The overall ranks of other algorithms were as follows: m$^2$wis+s (1.464), Cyclic-Fast (1.54), Solve (2.382), HtWIS (4.482), and HILS (4.89). In summary, DynLS consistently achieved the lowest rankings among all five compared algorithms, demonstrating its effectiveness and superior performance.

\section{Conclusion}\label{sec-con}
In this paper, we developed a dynamic local search framework for the MWIS problem named Dynamic Location Search (DynLS). The algorithm begins by kernelizing the original graph \( G \) using reduction rules. Different methods are applied to construct an initial solution based on the characteristics of the graph. Additionally, we introduced a scores-based adaptive vertex perturbation strategy based on four distinct scores to address the limitations of previous perturbation strategies (such as singular selection criteria and a fixed number of perturbed vertices),  with the region location mechanism implemented on the \textit{local graph}. To enhance solution space exploration, we developed a set of complex vertex exchange processes that leverage a reward mechanism to apply these exchange rules efficiently. This leads to ComLS, which can reopen new search spaces when the location mechanism becomes ineffective. Experimental results show that, compared to other state-of-the-art algorithms, DynLS can find more high-quality solutions across multiple datasets within 1000 seconds.

DynLS exhibits robust performance across a wide range of datasets; however, it does have certain limitations in a few specific instances. For example, it fails to achieve best solutions in four cases within the Osm dataset: ``d.o.c.3,'' ``hawaii3,'' ``kentucky3,'' and ``rhode-i.3''.  
{ Regarding these four instances, d.o.c.3 consists of 46221 vertices and 27729137 edges, hawaii3 has 28006 vertices and 49444921 edges, kentucky3 contains 19095 vertices and 59533630 edges, and rhode-i.3 has 15124 vertices and 12622219 edges. These instances exhibit a high density, which adversely affects the performance of DynLS. This may be due to the nature of the local search process, where the algorithm seeks improvement operations that satisfy specific criteria. In cases where a vertex has a high degree, the time required to search for valid neighboring vertices increases, thereby limiting effective vertex exchanges. Furthermore, the complex nature of the vertex exchange process means that many vertices may be involved. This is particularly disadvantageous for dense graphs, as changes to one vertex can impact numerous others, potentially leading the solution away from optimality. In contrast, the algorithm tends to explore the \textit{local graph} more thoroughly in sparser instances, resulting in improved performance on this type of graph.

DynLS struggles with certain instances, particularly those characterized by high edge density, suggesting that extensive neighborhoods can effectively capture the attributes of most vertices.
}
This insight prompts us to enhance DynLS by integrating Generative AI and reinforcement learning (RL) techniques. Dense vertex neighborhoods in graphs facilitate faster information propagation and more accurate feature extraction, which can be further improved by using Generative AI to generate enriched feature representations. {For example, a generative model can be developed to create a mapping relationship that links graph data with solution strategies. By utilizing smaller, edge-dense graph instances as the training set, we can employ exact algorithms to solve them within a reasonable time frame, generating solutions that serve as labels. In this process, the vertex weights and neighborhood relationships of the graph can be used as input features. The high edge density of the graph may enable the training process to capture features from a broader neighborhood of vertices, effectively learning from the data and enhancing the model’s predictive accuracy \cite{Tang2025}. The output of the model will be a probability distribution across all vertices, indicating the likelihood of each vertex being part of the optimal solution. Ultimately, applying the predictive capability of the model to large-scale, high-density graphs allows for the estimation of vertex probabilities. These probabilities can then be used as a basis for constructing initial solutions or as criteria for selecting vertices during optimization, thereby enhancing the algorithm's robustness.
}
{Additionally, reinforcement learning (RL) can be employed to enhance the search strategy. In this context, we can consider the states of all vertices—whether included in the solution or not—as the ``state'' in RL, while various vertex exchange strategies or perturbation mechanisms serve as the ``actions.'' Executing an action alters the state of the solution, and the change in weight before and after this adjustment is defined as the ``reward.'' The objective is to steer the solution towards global optimality by maximizing the reward. This integration seeks to bolster the capabilities of DynLS in efficiently tackling complex instances.
}

% Meanwhile, RL can optimize the search strategy by treating vertex inclusion as states, vertex exchanges as actions, and solution weight gains as rewards. This integration aims to improve DynLS's capability to address complex instances more efficiently, especially in dense graphs.

The implementation of the DynLS algorithm in intelligent transportation systems (ITSs) offers substantial potential. ITSs play a vital role in contemporary traffic management and optimization, enhancing traffic flow and safety. By utilizing the DynLS algorithm, ITSs can more effectively navigate complex traffic situations, including traffic congestion, route planning, and traffic decision-making. 
For instance, "MT-W-200" represents a graph derived from the real-world long-haul vehicle routing problem, focusing on optimizing transportation scheduling within Amazon's logistics. The objective is to identify a set of non-conflicting paths that maximizes the sum of their weights, thereby enhancing the efficiency of transportation scheduling. This scenario can be transformed into the MWIS problem \cite{dong2021new}, which can then be addressed using the DynLS algorithm. In this instance, the DynLS algorithm achieved an independent set with a total weight of 383933638 within 1000 seconds, outperforming other comparison algorithms.
\bibliographystyle{IEEEtran}
\bibliography{example}

\vfill

\end{document}